\documentclass{article}

\usepackage[preprint]{neurips_2023}

\usepackage[utf8]{inputenc} 
\usepackage[T1]{fontenc}    

\usepackage{amsmath}
\usepackage{amssymb}
\usepackage{amsthm}
\usepackage{amsfonts}       
\usepackage{algpseudocode}

\usepackage{booktabs}       
\usepackage{bbm}
\usepackage{graphicx} 
\usepackage{hyperref}       
\usepackage{url}            

\usepackage{nicefrac}       
\usepackage{microtype}      
\usepackage{MnSymbol}
\usepackage{mathbbol}
\usepackage{xcolor}         
\usepackage{wasysym}
\usepackage{url}            
\usepackage{subfig}

\usepackage{multirow}
\usepackage{lipsum}
\usepackage[normalem]{ulem}
\usepackage{tablefootnote}
\usepackage{cleveref}
\crefformat{footnote}{#2\footnotemark[#1]#3}

\newtheoremstyle{exampstyle}
  {\topsep} 
  {0pt} 
  {\itshape} 
  {} 
  {\bfseries} 
  {.} 
  {.5em} 
  {} 

\theoremstyle{exampstyle} 
\theoremstyle{exampstyle} \newtheorem{definition}{Definition}
\theoremstyle{exampstyle} \newtheorem{proposition}{Proposition}
\theoremstyle{exampstyle} \newtheorem{lemma}{Lemma}
\theoremstyle{exampstyle} \newtheorem{corollary}{Corollary}

\definecolor{myorange}{RGB}{205,115,0}

\newcommand{\blue}[1]{\textcolor{blue}{#1}}

\usepackage{enumitem}

\title{Recurrent Temporal Revision Graph Networks}

\date{November 2022}

\author{%
  Yizhou Chen\thanks{The three authors contributed equally to the paper.} \\
  Shopee Pte Ltd.,\\
  Singapore\\
  \And
  Anxiang Zeng\footnotemark[1]\\
  SCSE, Nanyang Technological  \\
  University, Singapore\\
  \And
  Guangda Huzhang\footnotemark[1]\\
  Shopee Pte Ltd.,\\
  Singapore\\
  \AND
  Qingtao Yu\\
  Shopee Pte Ltd.,\\
  Singapore\\
  \And
  Kerui Zhang\\
  Shopee Pte Ltd.,\\
  Singapore\\
  \And
  Yuanpeng Cao\\
  Shopee Pte Ltd.,\\
  Singapore\\
  \And
  Kangle Wu\\
  Shopee Pte Ltd.,\\
  Singapore\\
  \And
  Han Yu\\
  SCSE, Nanyang Technological  \\
  University, Singapore\\
  \And
  Zhiming Zhou\thanks{Corresponding author.  zhouzhiming@mail.shufe.edu.cn.}\\
  Shanghai University of Finance\\
  and Economics, China \\
}

\begin{document}
\maketitle
\begin{abstract}

Temporal graphs offer more accurate modeling of many real-world scenarios than static graphs. However, neighbor aggregation, a critical building block of graph networks, for temporal graphs, is currently straightforwardly extended from that of static graphs. It can be computationally expensive when involving all historical neighbors during such aggregation. In practice, typically only a subset of the most recent neighbors are involved. However, such subsampling leads to incomplete and biased neighbor information. To address this limitation, we propose a novel framework for temporal neighbor aggregation that uses the recurrent neural network with node-wise hidden states to integrate information from all historical neighbors for each node to acquire the complete neighbor information. 
We demonstrate the superior theoretical expressiveness of the proposed framework as well as its state-of-the-art performance in real-world applications. Notably, it achieves a significant $+9.6\%$ improvement on averaged precision in a real-world Ecommerce dataset over existing methods on 2-layer models.

\end{abstract}

\section{Introduction}

Many real-world applications (such as social network~\cite{tgn,ying2018graph}, recommender system~\cite{dai2016deep,sankar2020dysat}, traffic forecasting~\cite{li2017diffusion,zhang2018gaan} and crime analysis~\cite{jin2020addressing}) involve dynamically changing entities and increasing interactions, which can be naturally modeled as temporal graphs (also known as dynamic graphs) where nodes and edges appear and evolve over time. 
Compared with static graph solutions~\cite{kipf2016semi,velickovic2017graph,he2020lightgcn}, dynamic graphs can be more accurate in modeling node states and inferring future events. 


In modern graph representation learning, neighbor aggregation, i.e., generating higher-level node embeddings via referring to the embeddings of their neighbors~\cite{kipf2016semi, velickovic2017graph, li2019predicting}, is a critical common ingredient that helps more effectively leverage the connectivity information in the graphs. Practically, however, relatively hot nodes may have rapidly growing lists of neighbors, which makes involving all neighbors in neighbor aggregation computationally expensive. 
In static graphs, SAGE~\cite{hamilton2017inductive} suggests randomly sampling $d$ neighbors to approximate all neighbors. For temporal graphs, the prevalent approach is to consider only the most recent $d$ neighbors~\cite{tgn, pint}. 
But such subsampling of neighbors could lead to incomplete and biased information, as the neighbor information preceding the most recent $d$ neighbors is permanently lost. 


To mitigate the problem caused by neighbor subsampling, we propose an alternative solution to classic neighbor aggregation for temporal graphs, called recurrent temporal revision (RTR), which can be used as a standard building block (a network layer) for any temporal graph network. Each RTR layer incorporates a hidden state for each node to integrate the information from all its neighbors using a recurrent neural network (RNN). Such a hidden state will serve as the embedding of each node at this layer, attempting to embrace the complete neighbor information. 



The existence of such a hidden state implies that all historical neighbors have already been integrated through previous updates. So when a new edge event comes, we may build a message with the latest neighbor to update the hidden state through the RNN. On the other hand, taking the state updates of previous neighbors into account can also be beneficial, which could provide additional information, including new interactions of the neighbors~\cite{tgn} as well as updates in the neighbor states due to the influence of their own neighbors. 
In other words, re-inputting neighbors would enable revising their previously provided information. 
Therefore, we propose a module called temporal revision (TR) to collect state updates of previous neighbors to update the hidden state together with its latest interaction. To better capture the state updates, we additionally input the neighbors' revisions and state changes, besides their current states, into the temporal revision module. Subsampling of such state updates of neighbors can be less problematic, since the state of each neighbor has already been observed and integrated by the hidden state in the past. 




Towards further enhancement, we propose to introduce heterogeneity in temporal revision to boost its theoretical expressiveness. Heterogeneity can be naturally incorporated, since when a node collects revisions from its neighbors, the neighbors may, the other way around, collect the revision from the node (at a lower level), which is, however, directly accessible to the node itself. We thus introduce heterogeneity on such self-probings, by recursively identifying and marking such nodes as specialties  in revision calculation.
We theoretically show that such heterogeneous revision leads to improved expressiveness beyond Temporal-1WL~\cite{pint}, and practically verify its potential to promote effective information exchange among nodes. 

Our contributions can be summarized as follows: 
\begin{itemize}[leftmargin=10pt] 
    \item We propose a novel neighbor aggregation framework for temporal graphs to tackle the challenges arising from the necessary subsampling of neighbors. 
    \item We propose two expressiveness measures for temporal graph models in terms of temporal link prediction and temporal graph isomorphism test, and characterize their theoretical connection. We theoretically demonstrate our framework's superior expressiveness, as well as provide several new results on the expressiveness of other baselines. 
    \item We demonstrate in experiments that the proposed framework can significantly outperform state-of-the-art methods. In particular, we can consistently get improvements as the number of aggregation layers increases, while previous methods typically exhibit only marginal improvements~\cite{tgn, tgat}. 
    \item An Ecommerce dataset is provided along with this submission, which can be valuable for evaluating temporal graph models in real-world settings. 
\end{itemize}

\section{Preliminaries} 

\subsection{Temporal Graph Modeling} 

A temporal graph can be represented as an ordered list of timed events, which may include the addition, deletion, and update of nodes and edges. For simplicity, we will only discuss the handling of edge addition events in this paper, where newly involved nodes will be accordingly added and initialized. It can be easily extended to handle other events~\cite{tgn}. 
We denote the set of nodes and edges up to time $t$ as $V(t)$ and $E(t)$, respectively. 

Modeling of temporal graph typically generates node embeddings that evolve over time as graph events occur, which can be used by downstream tasks such as link prediction~\cite{chen2022gc,sankar2020dysat} and node classification~\cite{lu2021graph,deng2019graph}. The base node embeddings can have various sources, including direct node features~\cite{hamilton2017inductive,velickovic2017graph,ying2018graph}, learnable parameters~\cite{he2020lightgcn,wang2019neural}, generated by other modules such as an RNN~\cite{tgn, tgat, pint} or specifically designed module~\cite{dai2016deep,jodie,ma2020streaming}). Modern temporal graph methods~\cite{tgn, grugcn, pint} typically employ additional layers on top of the base embeddings to (further) leverage the connectivity information within the graph, generating higher-level node embeddings with a wider receptive field that encompasses relevant edges and nodes. 

\subsection{(Temporal) Neighbor Aggregation} 
\label{sec:temporal_agg} 

Among various approaches, neighbor aggregation is one of the commonly used techniques in various graph models as a critical ingredient, where the embedding of each node is updated by considering the embeddings of its neighbors. In the context of temporal graphs, the neighbor aggregation layer can be defined as follows~\cite{tgn,tgat,pint}: 
\begin{align}
    \mathbf{a}^k_u(t) & = \textsf{AGG}^{k} \Big( \Big\{ \big( \mathbf{h}^{k-1}_{v}(t), \Phi(e_{uv,t'}) \big) \;\big\mid\; e_{uv,t'} \in E(t) \Big\} \Big), 
    \label{equ:agg}
    \\ 
    \mathbf{h}^{k}_{u}(t) &= \textsf{COMBINE}^{k}\Big( \mathbf{h}^{k-1}_{u}(t), \mathbf{a}^k_u(t) \Big), 
    \label{equ:agg_update}
\end{align} 
where $t$ denotes the current time. $e_{uv,t'}$ denotes the event between node $u$ and node $v$ at time $t'(\leq t)$. $\Phi(e_{uv,t'})$ encodes the event feature, as well as its related time and sequence information, which may include the current time, the event time, the current number of interactions of $u$ and $v$, etc. $\textsf{AGG}^{k}$ 
and $\textsf{COMBINE}^{k}$ denote arbitrary learnable functions, which are typically shared at layer $k$. $\mathbf{a}^{k}_{u}(t)$ is the aggregated information from neighbors at the $k-1$ layer, while $\mathbf{h}^{k}_{u}(t)$ is the new embedding of node $u$ at layer $k$, combining its current embedding at the $k-1$ layer with the aggregated information $\mathbf{a}^{k}_{u}(t)$.
This can be viewed as a straightforward extension of classic graph convolution layers~\cite{kipf2016semi,velickovic2017graph,li2019predicting} with additional time or sequence information. 

\subsection{Expressiveness of Temporal Graph Models}


Expressiveness measure is crucial for identifying and comparing the capabilities of different models. While several studies have explored the expressiveness of (temporal) graph models~\cite{xu2018representation,grugcn,pint,gin,identity}, a widely accepted definition for the expressiveness of temporal graph models is currently lacking. To establish a solid groundwork, we propose the following formal definition of expressiveness in terms of temporal graph isomorphism test and temporal link prediction: 

\begin{definition} ({Temporal graph isomorphism indistinguishability}).
Given a pair of temporal graphs $\langle G_1(t), G_2(t) \rangle$, we say they are indistinguishable w.r.t.$\!$ a model $f$, if and only if there exists a bijective mapping between nodes of $G_1(t)$ and $G_2(t)$, such that, for each pair of nodes in the mapping, their embeddings generated by $f$ are identical at any time $t' \leq t$. 
\end{definition} 

\begin{definition} ({Temporal link prediction indistinguishability}).
Given a temporal graph $G(t)$ and two pairs of nodes $\langle (u_1, v_1), (u_2,v_2) \rangle$ therein representing two temporal link prediction problems, we say they are indistinguishable w.r.t. a model $f$, if and only if the embedding of $u_1$ and $u_2$, $v_1$ and $v_2$ generated by $f$ are identical at any $t' \leq t$. 
\end{definition}

\begin{definition} ({Temporal graph isomorphism expressiveness}).
If a model $f_A$ can distinguish any pair of temporal graphs that model $f_B$ can distinguish, while being able to distinguish certain graph pair that $f_B$ fails, we say $f_A$ is \emph{strictly more expressive} than $f_B$ in temporal graph isomorphism test. 
\end{definition}

\begin{definition} ({Temporal link prediction expressiveness}).
If a model $f_A$ can distinguish any pair of temporal link prediction problems that model $f_B$ can distinguish, while being able to distinguish certain pair of temporal link prediction problems that $f_B$ fails, we say $f_A$ is \emph{strictly more expressive} than $f_B$ in temporal link prediction. 
\end{definition}

We can similarly define equally expressive, more or equally expressive, and so on. Our definition of expressiveness differs from~\cite{pint}, which focuses on isomorphism of computation trees. It is also more concise compared with~\cite{grugcn}, avoiding explicit discussions of the identifiable set and its cardinality. 



Most existing results on the expressiveness of temporal graph models can be easily transferred to our definition. For example, message passing temporal graph networks (MP-TGN) are strictly more expressive in isomorphism test when the (temporal) neighbor aggregation layers are injective (referred to as IMP-TGN)~\cite{pint}; and IMP-TGN would be as expressive as temporal 1-Weisfeiler-Lehman (Temporal-1WL) test~\cite{pint,gin}, if with sufficient many layers.

\section{The Proposed Method}

\subsection{Recurrent Temporal Revision} 

\begin{figure}[t]
    \centering
    \subfloat[]{\includegraphics[height=39mm]{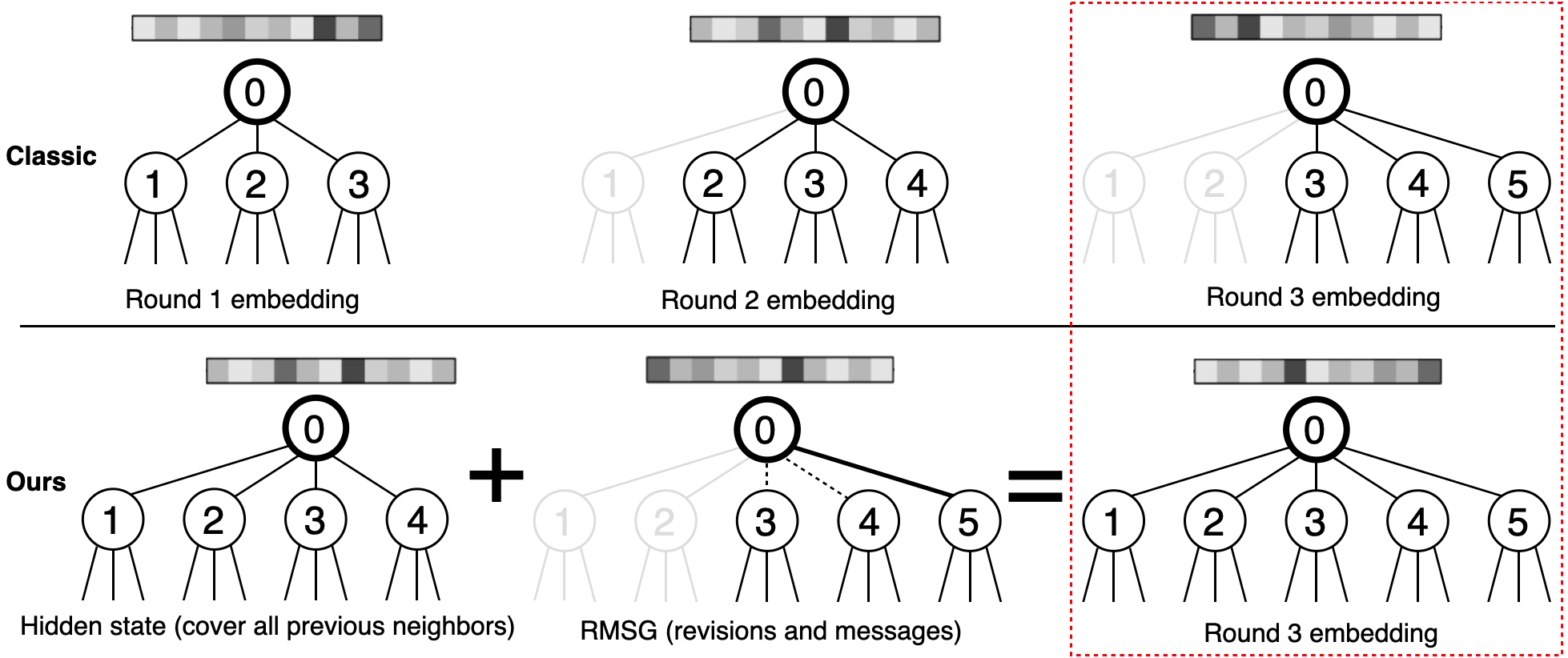}}
    \hspace{3mm}
    \subfloat[]{\includegraphics[height=39mm]{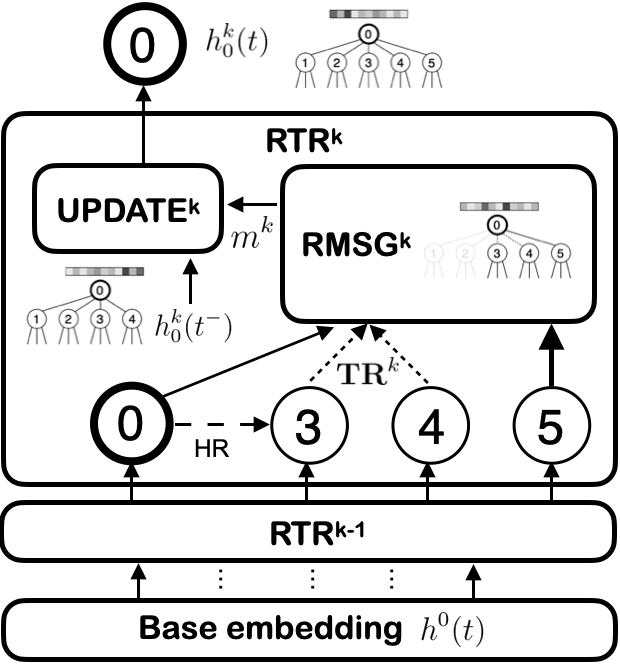}}
    \caption{(a) A comparison of our framework and other classic aggregations.(b) Illustration of our RTRGN. }
    \label{fig:staleness}
\end{figure}

In large graphs, having too many neighbors can pose practical computation challenges in neighbor aggregation. For static graphs, SAGE~\cite{hamilton2017inductive} suggests approximating neighbor information by randomly subsampling a fixed number ($d$) of neighbors. For temporal graphs, recent models~\cite{tgn,pint} typically involve the most recent $d$ neighbors to reduce the computation cost of temporal neighbor aggregation, which often yields better performance compared with uniform sampling~\cite{tgn}. But such subsampling of neighbors fails to provide complete neighbor information and biases it towards short-term behaviors. The effects are illustrated in Figure~\ref{fig:staleness}(a). 

Towards addressing this problem, we propose to introduce a hidden state $\mathbf{h}_{u}^{k}(t)$ in each aggregation layer $k>0$ for each node to integrate information from all its historical neighbors under a state transition model, typically recurrent neural network (RNN). Since events in the temporal graph, and hence the addition of neighbors, occur sequentially, such integration can be done naturally. 




Formally, when an event $e_{uv,t}$ occurs, we will need to update the hidden state $\mathbf{h}_{u}^{k}$ of $u$ (and also $v$). A straightforward implementation would be updating the hidden state with a message of the new event. However, taking the state updates of previous neighbors into account can also be beneficial, which can provide additional information including: the new interactions of neighbors, their state evolutionary directions, updates in the neighbor states due to the influence of their own neighbors. 
To achieve this, we introduce a module called temporal revision (TR) to collect state updates of previous neighbors, which will be later fused with the latest event to update the hidden state: 
\begin{equation}
    \mathbf{r}_{u}^{k}(t) = \textsf{TR}^{k}\Big(\Big\{\big(\mathbf{r}_{v}^{k-1}(t),\Delta\mathbf{h}_{v}^{k-1}(t),\mathbf{h}_{v}^{k-1}(t),\Phi(e_{uv,t'})\big) \;\big\mid\; e_{uv,t'}\in E(t^-)\Big\}\Big), 
    \label{equ:revision_recursion no H}
\end{equation}
where $\textsf{TR}^{k}$ is an arbitrary learnable function shared at layer $k$, $t^-$ represents the time right before $t$. We collect not only the newest states of neighbors $\mathbf{h}_{v}^{k-1}(t)$ at the $k-1$ layer, but also their state changes $\Delta\mathbf{h}_{v}^{k-1}(t)=\mathbf{h}_{v}^{k-1}(t)-\mathbf{h}_{v}^{k-1}(t^{-})$ and their revisions $\mathbf{r}_{v}^{k-1}(t)$, as well as the event information $\Phi(e_{uv,t'})$, to help the model better extract the state update information. We set the base case $\mathbf{r}_{u}^{0}(t)=0$. When requesting the embedding of a node at the $0$-th layer, we refer to its base embedding\footnote{Please refer to Appendix~\ref{Appendix:A} for more detailed discussions of the base embedding.}. 


In subsequent ablation studies, we show that retaining $\Delta\mathbf{h}_{v}^{k-1}(t)$ alone, while removing $\mathbf{h}_{v}^{k-1}(t)$, will not lead to degeneration of performance, which means the effectiveness in aggregation can rely solely on the changes of states, while not necessarily on the neighbor states themselves. Temporal revision is hence substantially different from classic temporal aggregations~\cite{tgat,tgn,grugcn,pint}. 


To encode all neighbor update information, we build a state revising message (RMSG) via: 
\begin{flalign}    
    \mathbf{m}_{u}^{k}(t) = \textsf{RMSG}^{k}\big(
    \mathbf{h}_{u}^{k-1}(t),\mathbf{h}_{v}^{k-1}(t),
    \mathbf{r}_{u}^{k}(t),
    \Phi(e_{uv,t})\big),
    \label{msg}
\end{flalign}
where $\textsf{RMSG}^{k}$ is an arbitrary learnable function shared at layer $k$, integrating the node embedding at the $k-1$ layer $\mathbf{h}_{u}^{k-1}(t)$, the embedding of the new neighbor at the $k-1$ layer $\mathbf{h}_{v}^{k-1}(t)$, the collected neighbor state update information $\mathbf{r}_{u}^{k}(t)$, as well as the event information $\Phi(e_{uv,t})$.\footnote{When there is no event at time $t$ for $u$, but requesting the message, we set $\mathbf{h}_{v}^{k-1}(t)=\mathbf{0},\Phi(e_{uv,t})=\mathbf{0}$ in \eqref{msg}.$\!\!$}


Then the proposed recurrent temporal revision (RTR) layer can be formulated as: 
\begin{equation} 
   \mathbf{h}^{k}_{u}(t)=\textsf{UPDATE}^k\big(\mathbf{h}^{k}_{u}(t^-), \mathbf{m}_{u}^{k}(t) \big), 
    \label{equ:TRGN_recursion} 
\end{equation}
where the state revising message is consumed to update the hidden state $\mathbf{h}^{k}_{u}$ from time $t^-$ to $t$ via $\textsf{UPDATE}^k$, an learnable function shared at layer $k$ responsible for hidden state transition. For all RTR layers, we initialize $\mathbf{h}_{u}^k(0)$ as zero vector for all nodes. 



\subsection{Heterogeneous Revision}

Heterogeneous message passing can improve the expressiveness of static GNN~\cite{identity}, which can be naturally integrated in our revision computation. The term $\mathbf{r}_{u}^{k}(t)$ recursively gathers revision from previously interacted nodes, and they may in turn gather the revision of node $u$ (at a lower-level), which is, however, directly accessible to $u$ itself. Therefore, we may adopt a different treatment for such self-probing calculations during recursive aggregation. To achieve this, we may mark all nodes in the recursive computation path as specialities. 

Formally, we define the heterogeneous revision as: 
\begin{equation}
\hspace{-2mm}
    \,\mathbf{r}_{u}^{k}(t,\mathcal{S})\!=\!\textsf{TR}^{k}\Big(\Big\{\big(\mathbf{r}_{v}^{k-1}(t,\mathcal{S}\cup\{u\}),\Delta\mathbf{h}_{v}^{k-1}(t),\mathbf{h}_{v}^{k-1}(t),\Phi(e_{uv,t'}),\mathbbm{1}[v\in\mathcal{S}] \big) \;\Big\mid\; e_{uv,t'} \!\in E(t^-) \Big\}\Big), 
    \label{equ:revision_recursion} 
\end{equation}
where we introduce an extra argument $\mathcal{S}$ for revision, which is a set that is initialized as empty $\varnothing$ and recursively adds each $u$ encountered along the computation path. $\mathbbm{1}$ stands for the indicator function, which triggers the difference between nodes in $\mathcal{S}$ and the rest, bringing about heterogeneity in the revisions. To switch to heterogeneous revision, we replace $\mathbf{r}_{u}^{k}(t)$  in \eqref{msg} with $\mathbf{r}_{u}^{k}(t,\varnothing)$. We set the base case $\mathbf{r}_{u}^{0}(t,\mathcal{S})=\mathbf{0}$ for any $u$ and $\mathcal{S}$. 

We illustrate the RTR layer in Figure~\ref{fig:staleness}(b) and 
provide the implementation details in Appendix~\ref{Appendix:A}. 





\begin{table}[h]
    \caption{Summary of the theoretical results. Our results are presented in black.}
    \centering
    \hspace{-3mm}
    \scalebox{0.80}{
    \begin{tabular}{l|l}
    \hline 
    \multicolumn{1}{c|}{Expressiveness (\textbf{Link Prediction})}   
    & \multicolumn{1}{c}{Expressiveness (\textbf{Isomorphism Test})} \\ \hline 
    \multicolumn{2}{c}{\;\;\;\;\;\;$f_A\succ_{is} f_B\Rightarrow f_A\succ_{lp} f_B$ (Prop.~\ref{prop:1}) \;\;\;\;\;\; $f_A\cong_{is} f_B\Rightarrow f_A\cong_{lp} f_B$ (Corollary~\ref{corollary:equal to equal})} \\ \hline 
    \textcolor{brown}{PINT-pos $\succ_{lp}$ PINT $\succ_{lp}$ MP-TGNs~\cite{pint}} & \textcolor{brown}{ PINT $\cong_{is}$ Temporal-1WL~\cite{pint}} \\
    Time-then-IMP $\cong_{lp}$ PINT $\cong_{lp}$ Temporal-1WL (Prop.~\ref{prop:2}) & Time-then-IMP $\cong_{is}$ PINT $\cong_{is}$ Temporal-1WL (Prop.~\ref{prop:2}) \\
    RTRGN $\succ_{lp}$ PINT (Prop.~\ref{prop:3}) & RTRGN $\succ_{is}$ PINT (Prop.~\ref{prop:3}) \\
    PINT-pos can be false positive (Prop.~\ref{prop:4}) & PINT-pos can be false positive (Prop.~\ref{prop:4}). \\ \hline
    \multicolumn{2}{c}{Most expressive: RNN + Temporal Relational Pooling (Prop.~\ref{prop:max expressive} in Appendix~\ref{Appendix: B})} \\ \hline
    \end{tabular}}
    \label{tab:expressive}
\end{table}

\begin{figure}[h]
    \centering
    \includegraphics[width=140mm]{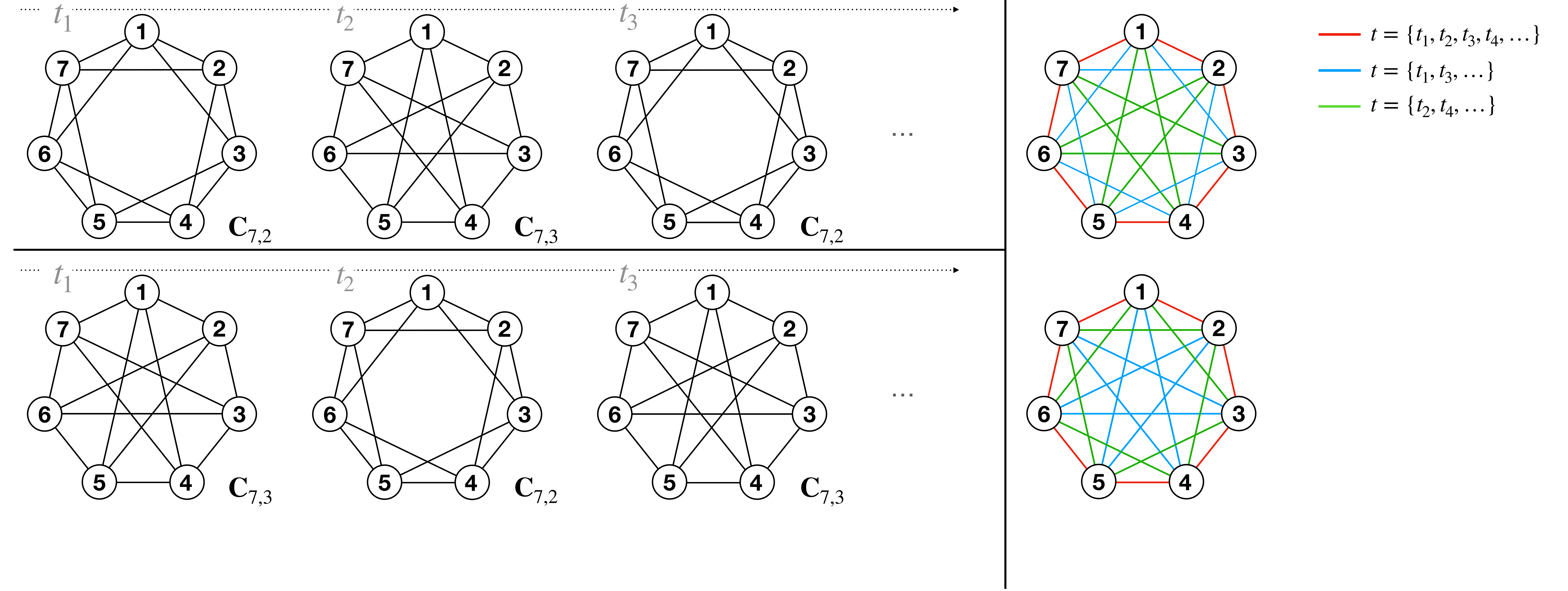}
    \caption{A synthetic graph isomorphism test task called Oscillating CSL, where only RTRGN is expressive enough to give the correct result. The task is to distinguish the top and bottom temporal graphs, both of which oscillate but in different order between two regular structures called circular skip links (CSL). The right side shows the stacked view of the two temporal graphs.} 
    \label{fig:OscillatingCSL}
\end{figure}

\subsection{Model Expressiveness}

For model expressiveness, we assume our RTR layers are built on top of a base RNN. For a fair comparison, it is assumed the base RNNs for all models are the same and of the best expressiveness, as suggested in~\cite{grugcn,pint}. 
We focus on RTR layers whose (learnable) functions in \eqref{msg} - \eqref{equ:revision_recursion} are injective, which can be achieved by adopting similar approaches as in~\cite{pint,gin}. We refer to such a model as RTRGN. Besides, we assume node and edge attributes come from a finite set as~\cite{relationalpooling,grugcn,pint}. 

We will mainly compare RTRGN with PINT~\cite{pint}, which is an instance of IMP-TGN that can be equally expressive as Temporal 1-WL in isomorphism test. PINT also offers augmented positional features (referred to as PINT-pos) designed to enhance its expressiveness. 

Our first theoretical result demonstrates that the temporal graph isomorphism test is inherently more challenging than temporal link prediction. All proofs are deferred to Appendix~\ref{Appendix: B}. 

\begin{proposition}
    Being more expressive in temporal graph isomorphism test implies being more expressive in temporal link prediction ($\!f_{\!A}\!\succ_{is}\!\!f_{\!B}\!\Rightarrow\!\!f_{\!A}\!\succ_{lp}\!\!f_{\!B}$) but not vice versa ($\!f_{\!A}\!\succ_{lp}\!\!f_{\!B}\!\nRightarrow\!\!f_{\!A}\!\succ_{is}\!\!f_{\!B}$).$\!\!\!\!\!\!\!$
    \label{prop:1}
\end{proposition}

The first part of this proposition can be proved by reducing the temporal link prediction problem into a temporal graph isomorphism test of the current graph adding the two temporal links separately. The second part of this proposition is proved by constructing a pair of two models: one is more expressive than the other in temporal link prediction, while being not more expressive in temporal graph isomorphism test. The proof and the construction of the counter examples inside suggest that temporal graph isomorphism test is a more challenging task than temporal link prediction, as it requires global distinguishability rather than just local distinguishability of the involved nodes. 

Next, we demonstrate that Time-then-IMP, the strongest among the three classes of temporal graph models suggested in~\cite{grugcn} with the graph module of its time-then-graph scheme being injective message passing (IMP) to ensure 1WL expressiveness~\cite{gin}, is equally expressive as PINT and Temporal-1WL in both temporal graph isomorphism test and temporal link prediction. 

\begin{proposition} 
Time-then-IMP is equally expressive as PINT and Temporal-1WL in both temporal graph isomorphism test and temporal link prediction. 
\label{prop:2}
\end{proposition}

The proof establishes that all these models are equally expressive as Temporal-1WL in both temporal link prediction and temporal graph isomorphism test. The proposition suggests that Time-then-IMP, although utilizes a distinct form of base RNN to encode all historical features of nodes and edges (which comes at the cost of significantly more space needed to embed each edge), its expressiveness is ultimately equivalent to the Temporal-1WL class, which includes models like PINT. 

We next show that RTRGN is more expressive than PINT, as well as the Temporal-1WL class: 

\begin{proposition}
RTRGN is strictly more expressive than the Temporal-1WL class represented by PINT in both temporal graph isomorphism test and temporal link prediction. 
\label{prop:3}
\end{proposition}

As we have shown, the expressiveness of all the previously mentioned models is bounded by the Temporal-1WL, which, although powerful, can fail in certain cases. To illustrate this, we construct a temporal graph isomorphism test task called Oscillating CSL (Figure~\ref{fig:OscillatingCSL}) where all Temporal-1WL models, including Time-then-IMP, fail. The difficulty of the task stems from the regular structure of the graph snapshots and stacked graphs at each timestamp, which cannot be distinguished by Temporal-1WL tests. 
RTRGN excels in this task due to the heterogeneity introduced in its revision computation, which enables its expressiveness to go beyond the limit of Temporal-1WL. 

\begin{proposition}
PINT-pos can be false positive in both temporal graph isomorphism test and temporal link prediction, e.g., may incorrectly classify isomorphic temporal graphs as non-isomorphic.$\!\!\!\!\!$ 
\label{prop:4}
\end{proposition}

On the other hand, we observe that PINT-pos can produce false positives (unlike the other models). Specifically, there are cases where PINT-pos mistakenly categorizes isomorphic temporal graphs as non-isomorphic, and incorrectly distinguishes two temporal links that should not be distinguishable. 
It fails essentially because its positional feature is permutation-sensitive, while all other mentioned models are permutation-invariant, which is a desirable property in graph model design~\cite{grugcn, relationalpooling}. This further highlights the superiority of heterogeneous revision over augmented positional features. 

Finally, we formulate a most expressive model that can always output the groundtruth (Appendix~\ref{Appendix: B}). 

\section{Related Works}

Early approaches in dynamic graphs learning focused on discrete-time dynamic graphs (DTDGs), which consist of a sequence of graph snapshots usually sampled at regular intervals~\cite{pareja2020evolvegcn,goyal2020dyngraph2vec,you2022roland}. 
They employed techniques such as stacking graph snapshots and then applying static methods~\cite{liben2003link,ahmed2016efficient, ibrahim2015link,sharan2008temporal} and assembling snapshots into tensors for factorization~\cite{yu2017link,ma2019embedding}. 
Among them, ROLAND~\cite{you2022roland} introduced hierarchical hidden states, but each state update is followed by a heavy finetuning of the entire model, which can be computationally expensive when applied to continuous-time dynamic graphs (CTDGs). In comparison, our framework is directly applied to CTDGs, handling the lists of timed events, and can be much more efficient. 


DTDGs can always be converted to CTDGs without information loss. However, the high computation cost poses a major obstacle for applying most DTDG methods to CTDGs. More and more research focuses on CTDGs.
Sequence-based methods such as DeepCoevolve~\cite{dai2016deep}, Jodie~\cite{jodie}, and Dyrep~\cite{dyrep} utilize RNN to capture the evolution of node representations, while some other works~\cite{nguyen2018continuous,mahdavi2018dynnode2vec} employ random walks to model CTDG, including CAWN~\cite{cawn}. NAT~\cite{nat} caches the high-order neighbor information for computing joint neighbor features. 

TGAT~\cite{tgat} and TGN~\cite{tgn} were the first to introduce GNN for CTDGs. Subsequently, combining GNNs with RNNs for temporal graph embedding generation became a mainstream idea. PINT~\cite{pint} introduced MLP-based injective aggregators to replace attention-based aggregators in TGN. GRU-GCN~\cite{grugcn} analyzed three strategies to combine GNN with RNN and classified them into time-and-graph (e.g.,~\cite{seo2018structured,li2019predicting,chen2022gc}), time-then-graph (e.g.,~\cite{tgat,tgn}), and graph-then-time frameworks (e.g.,~\cite{pareja2020evolvegcn,goyal2018dyngem,sankar2020dysat}). 

Our heterogeneous revision leads to enhanced expressiveness. Similar ideas were adopted in static GNNs. For example,~\cite{identity} introduced heterogeneity in the root node during message passing. 

\section{Experiments}

\subsection{Temporal Link Prediction on Benchmark Datasets}


We evaluate the proposed RTRGN with the temporal link prediction task on 8 real-world datasets (7 public datasets as well as a private Ecommerce dataset that will be released along with this submission), and compare it against related baselines, including Dyrep~\cite{dyrep}, Jodie~\cite{jodie}, TGAT~\cite{tgat}, TGN~\cite{tgn}, CAWN~\cite{cawn}, NAT~\cite{nat}, GRU-GCN~\cite{grugcn} and PINT~\cite{pint}. The evaluation is conducted in two settings: the transductive which predicts links involving nodes observed during training, and the inductive which predicts links involving nodes that are never observed during training. 

\begin{table}[t]
\centering
\caption{Average Precision (AP, \%) on temporal link prediction tasks with 1-layer models ($k=1$). \textcolor{violet}{First} and \textcolor{myorange}{second} best-performing methods are highlighted. Results are averaged over 10 runs.} 
\scalebox{0.67}{
\begin{tabular}{c|c|ccccccccc}
\hline
& Model              & MovieLens     & Wikipedia     & Reddit        & SocialE.-1m   & SocialE.      & UCI-FSN       & Ubuntu        & Ecommerce        \\ \hline
\parbox[t]{2mm}{\multirow{7}{*}{\rotatebox[origin=c]{90}{Transductive}}} 
& Dyrep              & $78.07\pm0.2$ & $95.81\pm0.2$ & $98.00\pm0.2$ & $81.74\pm0.4$ & $88.95\pm0.3$ & $53.67\pm2.1$ & $84.99\pm0.4$ & $62.98\pm0.3$ \\
& Jodie              & $78.42\pm0.4$ & $96.15\pm0.4$ & $97.29\pm0.1$ & $70.20\pm1.8$ & $81.83\pm1.1$ & $86.67\pm0.7$ & \textcolor{myorange}{$\mathbf{91.32\pm0.2}$} 
                                                                                                                                     & $68.70\pm0.6$ \\
& TGAT               & $66.64\pm0.5$ & $95.45\pm0.1$ & $98.26\pm0.2$ & $51.97\pm0.6$ & $50.75\pm0.2$ & $73.01\pm0.6$ & $81.69\pm0.3$ & $65.13\pm0.5$ \\
& TGN           & $84.27\pm0.5$ & $97.58\pm0.2$ & $98.30\pm0.2$ & \textcolor{myorange}{$\mathbf{90.01\pm0.3}$} 
                                                                                     & $91.06\pm1.5$ & $86.58\pm0.3$ & $90.18\pm0.1$ & $83.69\pm0.5$ \\
& CAWN               & $82.10\pm0.4$ & $98.28\pm0.2$ & $97.95\pm0.2$ & $84.62\pm0.4$ & $83.49\pm0.7$ & $90.03\pm0.4$ & -             & \textcolor{myorange}{$\mathbf{84.58\pm0.1}$} \\
& NAT                & $75.85\pm2.6$ & $98.27\pm0.1$ & \textcolor{myorange}{$\mathbf{98.71\pm0.2}$} 
                                                                     & $85.84\pm0.4$ & \textcolor{myorange}{$\mathbf{91.35\pm0.6}$} 
                                                                                                     & \textcolor{myorange}{$\mathbf{92.80\pm0.4}$}
                                                                                                                     & $90.86\pm0.5$ 
                                                                                                                                     & $81.69\pm0.2$ \\
& GRU-GCN            & \textcolor{myorange}{$\mathbf{84.70\pm0.2}$} 
                                     & $96.27\pm0.2$ & $98.09\pm0.3$ & $86.02\pm1.2$ & $84.22\pm2.5$ & $90.51\pm0.4$ & $86.52\pm0.2$ & $79.52\pm0.4$ \\
& PINT               & $83.25\pm0.9$ & \textcolor{myorange}{$\mathbf{98.45\pm0.1}$} 
                                                     & $98.39\pm0.1$ & $80.97\pm2.2$ & $84.94\pm3.4$ & $92.68\pm0.5$ & $89.59\pm0.1$ & $80.36\pm0.3$ \\
& RTRGN (ours)   & \textcolor{violet}{$\mathbf{86.56\pm0.2}$}
                                     & \textcolor{violet}{$\mathbf{98.56\pm0.2}$} 
                                                     & \textcolor{violet}{$\mathbf{99.00\pm0.1}$}
                                                                     & \textcolor{violet}{$\mathbf{92.20\pm0.1}$}
                                                                                     & \textcolor{violet}{$\mathbf{94.02\pm0.3}$} & \textcolor{violet}{$\mathbf{96.43\pm0.1}$} 
                                                                                                                     & \textcolor{violet}{$\mathbf{96.69\pm0.1}$}
                                                                                                                                     & \textcolor{violet}{$\mathbf{88.05\pm0.4}$} \\ \hline
\parbox[t]{2mm}{\multirow{7}{*}{\rotatebox[origin=c]{90}{Inductive}}}   
& Dyrep              & $74.47\pm0.3$ & $94.72\pm0.2$ & $97.04\pm0.3$ & $75.58\pm2.1$ & $88.49\pm0.6$ & $50.43\pm1.2$ & $71.49\pm0.4$ & $53.59\pm0.7$ \\
& Jodie              & $74.61\pm0.3$ & $95.58\pm0.4$ & $95.96\pm0.3$ & $73.32\pm1.4$ & $79.58\pm0.8$ & $71.23\pm0.8$ & $83.81\pm0.3$ & $61.69\pm0.6$ \\
& TGAT               & $66.33\pm0.4$ & $93.82\pm0.3$ & $96.42\pm0.3$ & $52.17\pm0.5$ & $50.63\pm0.1$ & $66.89\pm0.4$ & $78.92\pm0.5$ & $64.58\pm0.8$ \\
& TGN           & $82.07\pm0.2$ & $97.05\pm0.2$ & $96.87\pm0.2$ & \textcolor{myorange}{$\mathbf{88.70\pm0.5}$} 
                                                                                     & \textcolor{myorange}{$\mathbf{89.06\pm0.7}$} 
                                                                                                     & $81.53\pm0.2$ & $81.81\pm0.4$ & $81.81\pm0.5$ \\
& CAWN               & $74.50\pm0.5$ & $97.70\pm0.2$ & $97.37\pm0.3$ & $75.39\pm0.4$ & $81.55\pm0.5$ & $89.65\pm0.4$ & -             & \textcolor{myorange}{$\mathbf{83.30\pm0.4}$} \\
& NAT                & $77.56\pm0.6$ & \textcolor{myorange}{$\mathbf{97.74\pm0.4}$} 
                                                     & $97.19\pm0.7$ & $85.16\pm1.2$ & $85.32\pm3.6$ & $87.83\pm0.5$ & $81.69\pm0.9$ & $76.85\pm0.2$ \\
& GRU-GCN            & \textcolor{myorange}{$\mathbf{82.81\pm0.2}$} 
                                     & $93.50\pm0.2$ & $96.38\pm0.3$ & $83.76\pm1.8$ & $79.21\pm4.5$ & $85.45\pm0.6$ & $73.71\pm0.5$ & $78.54\pm0.2$ \\
& PINT               & $81.49\pm0.6$ & $97.29\pm0.1$ & \textcolor{myorange}{$\mathbf{97.69\pm0.1}$}
                                                                     & $77.45\pm1.9$ & $71.86\pm3.6$ & \textcolor{myorange}{$\mathbf{90.25\pm0.3}$} & \textcolor{myorange}{$\mathbf{85.74\pm0.2}$} & $77.15\pm0.5$  \\
& RTRGN (ours)   & \textcolor{violet}{$\mathbf{85.03\pm0.2}$} 
                                     & \textcolor{violet}{$\mathbf{98.06\pm0.2}$}
                                                     & \textcolor{violet}{$\mathbf{98.26\pm0.1}$} 
                                                                     & \textcolor{violet}{$\mathbf{91.73\pm0.7}$} 
                                                                                     & \textcolor{violet}{$\mathbf{92.47\pm0.3}$}
                                                                                                     & \textcolor{violet}{$\mathbf{94.26\pm0.1}$} 
                                                                                                                     & \textcolor{violet}{$\mathbf{92.95\pm0.1}$} 
                                                                                                                                     & \textcolor{violet}{$\mathbf{86.91\pm0.5}$} \\ \hline
\end{tabular}}
\label{tab:main ap}
\end{table}

\begin{table}[ht]
\centering
\begin{minipage}{.6\linewidth}
\vspace{-5mm}
\caption{Average Precision (\%) on temporal link prediction tasks with 2-layer models ($k=2$).} 
\scalebox{0.75}{
\begin{tabular}{c|c|cccc}
\hline
   & Model           & MovieLens     & Wikipedia     & UCI-FSN       & Ecommerce  \\ \hline
\parbox[t]{2mm}{\multirow{4}{*}{\rotatebox[origin=c]{90}{Trans}}} 
& TGN           & $85.33\pm0.5$ & $98.38\pm0.2$ & $86.69\pm0.3$ & $85.59\pm0.2$ \\
& NAT                & $80.04\pm1.2$ & $98.15\pm0.1$ & $93.03\pm0.1$ & $82.32\pm0.8$\\
& PINT               & $84.18\pm0.2$ & $98.33\pm0.1$ & $93.48\pm0.4$ & $82.35\pm0.7$\\
& RTRGN              & \textcolor{violet}{$\mathbf{94.90\pm0.4}$}
                                     & \textcolor{violet}{$\mathbf{98.79\pm0.2}$} 
                                                     & \textcolor{violet}{$\mathbf{98.04\pm0.1}$} 
                                                                     & \textcolor{violet}{$\mathbf{95.28\pm0.1}$} \\ \hline
\parbox[t]{2mm}{\multirow{4}{*}{\rotatebox[origin=c]{90}{Ind}}}
& TGN           & $83.43\pm0.5$ & $97.85\pm0.2$ & $83.58\pm0.3$ & $84.31\pm0.2$ \\
& NAT                & $72.23\pm1.6$ & $97.73\pm0.2$ & $88.11\pm0.6$ & $78.72\pm1.8$\\
& PINT               & $81.95\pm0.5$ & $96.39\pm0.1$ & $91.23\pm0.3$ & $80.30\pm1.0$\\
& RTRGN              & \textcolor{violet}{$\mathbf{93.37\pm0.3}$} 
                                     & \textcolor{violet}{$\mathbf{98.28\pm0.2}$}
                                                     & \textcolor{violet}{$\mathbf{95.62\pm0.1}$} 
                                                                     & \textcolor{violet}{$\mathbf{92.87\pm0.1}$} \\ \hline
\end{tabular}}
\label{tab:main ap 2-layer}
\end{minipage}
\hspace{1mm}\hfill
\begin{minipage}{.35\linewidth}
\centering
\includegraphics[width=45mm]{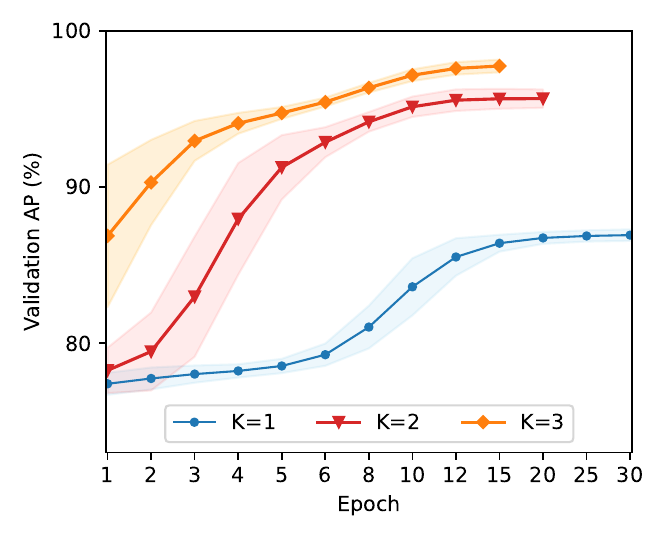}
\vspace{-1mm}
    \captionof{figure}{The convergence trend for different $k$.}
    \label{fig:convergence}
\end{minipage}
\end{table}

The official implementations of different baselines can vary in their experiment settings, especially on the \emph{dataset preparation} and the \emph{negative edge sampling}. To ensure a fair and consistent comparison, we calibrate the experiment settings of all baselines with the one used by TGN, which is the same as TGAT and PINT, but differs somewhat from other works such as CAWN and NAT. The latter may lead to inflated results, with inductive performance surpassing transductive performance, potentially reaching over $99\%$ accuracy in CAWN and NAT. We set $d=10$ for our experiments. More details about the datasets, the baselines, and the experiment settings can be found in Appendix~\ref{Appendix: Details of Experiment Settings}. 

The results for the 1-layer setting ($k=1$) are presented in Table~\ref{tab:main ap}. It is evident that our proposed RTRGN surpasses all baselines by a significant margin, suggesting a substantial improvement. 
To delve deeper into the problem, we further examine the results for the 2-layer setting on selected datasets (Table~\ref{tab:main ap 2-layer}), where we compare RTRGN against baselines that have demonstrated potentially better performance than other baselines in the 1-layer setting. As we can see, the performances of the baselines show limited improvement when transitioning from the 1-layer to 2-layer setting, aligning with previous findings~\cite{tgat,tgn}. By contrast, RTRGN consistently exhibits \textit{significant} improvements in the 2-layer setting: the average gap of transductive AP between RTRGN and the second best method is $+6.1\%$ among 4 selected datasets. The largest gap appears in our Ecommerce dataset, which is $+9.6\%$. The gap on Wikipedia is small probably due to a 1-layer model is already sufficient to give strong performance. 
This indicates that layer design is important for temporal graph networks to more effectively leverage the power of multiple layers. Another advantage we found in experiments is that the training converges faster for larger $k$, as shown in Figure~\ref{fig:convergence}. 


\subsection{Temporal Graph Isomorphism Test} 

To validate the superior expressiveness of the proposed RTRGN against existing methods, we conduct experiments of temporal graph isomorphism tests on two synthetic datasets. We run 4-layer models on testing graphs, where at each time step, we compute and store the embedding for each node. To check the isomorphism between two temporal graphs under a model, we gather the lists of generated embeddings (stacking in time for each node) of the two graphs and verify whether there exists a bijective mapping such that the two lists are identical (via sort and scan). If the two lists are identical, we conclude that the model outputs isomorphic (negative); otherwise, non-isomorphic (positive). 

\textbf{Oscillating CSL} is a synthetic temporal graph isomorphism test task, which is a commonly used benchmark for assessing the expressiveness of GNNs~\cite{relationalpooling,grugcn}. Each graph in this task is composed of a sequence of graph snapshots (of length $6$) that correspond to the oscillation between 2 different prototypes randomly selected from $\{\mathbf{C}_{11,2},\ldots,\mathbf{C}_{11,5}\}$. Here $\mathbf{C}_{N,s}$ denotes a circular skip link (CSL) graph with $N$ nodes and skip length $s$. The goal is to determine whether the temporal graph follows the pattern $\mathbf{C}_{N,s_1}\!\to\mathbf{C}_{N,s_2}$ is isomorphic to its reversed counterpart $\mathbf{C}_{N,s_2}\!\to\mathbf{C}_{N,s_1}$ (Figure~\ref{fig:OscillatingCSL}). 

\begin{figure}[h]
    \centering
    \includegraphics[width=135mm]{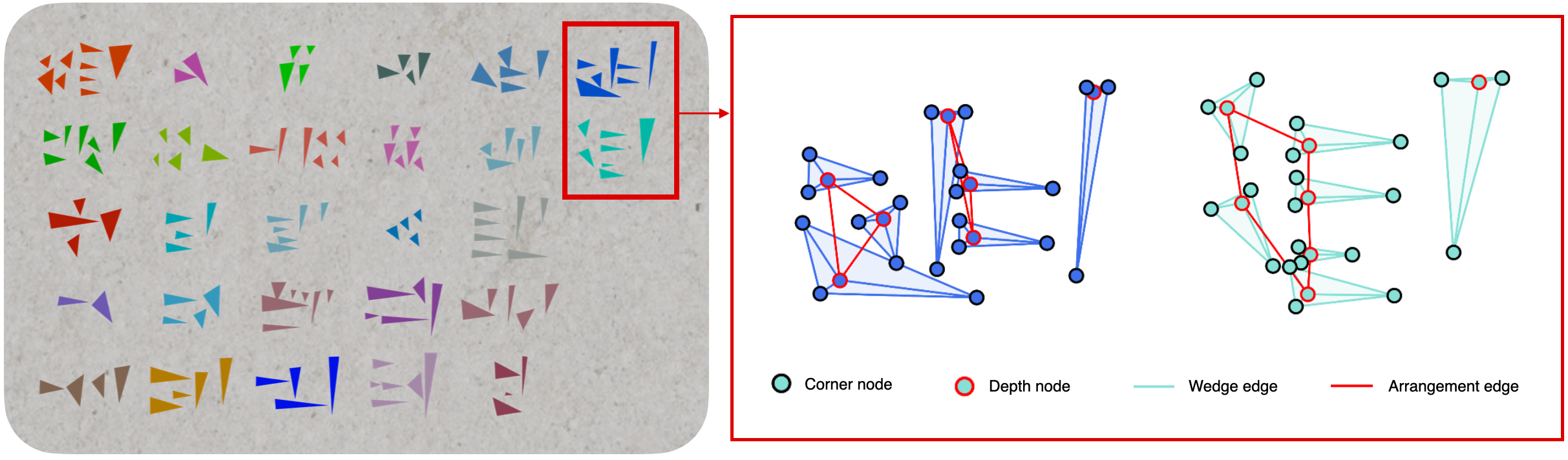}
    \caption{\textit{(left)} An illustration of the Hittite cuneiform signs. \textit{(right)} An example where TGN and GRU-GCN fails to output non-isomorphic, due to the special structure within the arrangement edges, i.e., two triangles versus a hexagon.} 
    \label{fig:cun_demo} 
\end{figure} 

\textbf{Cuneiform Writing} is a modified version of the Cuneiform graph dataset~\cite{kriege2018recognizing} (Figure~\ref{fig:cun_demo}), which consists of graphs representing Hittite Cuneiform signs. The dataset was originally derived from tablets, where individual wedges along with their arrangements were automatically extracted. 
We draw out arrangement edges with short spatial distances, while ensuring that each depth node has at most 2 arrangement edges. 
To transform static graphs into temporal graphs, we adopt sequential edge generation that adds arrangement edges first, then the wedge edges. For each task, we randomly select 2 cuneiform signs with an equal number of nodes. Refer to Appendix~\ref{Appendix: Details of Experiment Settings} for more details. 

\begin{table}[h]
\centering
\caption{Isomorphism test results. Non-isomorphic is regarded as positive. TNR is true negative rate.} 
\scalebox{0.75}{
\begin{tabular}{c|ccccc|ccccc}
\hline
          & \multicolumn{5}{c|}{Oscillating CSL}               & \multicolumn{5}{c}{Cuneiform Writing}               \\ \hline
          & Precision & Recall  & TNR     & Accuracy & AUC     & Precision & Recall  & TNR     & Accuracy & AUC      \\ \hline
TGN       & $0.00$    & $0.00$  & $\textcolor{violet}{\mathbf{100.0}}$ & $16.67$  & $50.0$  & $\textcolor{violet}{\mathbf{100.0}}$   & $74.19$ & $\textcolor{violet}{\mathbf{100.0}}$ & $79.80$  & $86.98$  \\
GRU-GCN   & $0.00$    & $0.00$  & $\textcolor{violet}{\mathbf{100.0}}$ & $16.67$  & $50.0$  & $\textcolor{violet}{\mathbf{100.0}}$   & $95.44$ & $\textcolor{violet}{\mathbf{100.0}}$ & $96.43$  & $97.61$  \\
PINT-pos  & $83.33$   & $\textcolor{violet}{\mathbf{100.0}}$ & $0.00$  & $83.33$  & $50.0$  & $82.03$   & $\textcolor{violet}{\mathbf{100.0}}$ & $21.09$ & $82.85$  & $60.93$  \\
RTRGN (w/o HR)   & $0.00$    & $0.00$  & $\textcolor{violet}{\mathbf{100.0}}$ & $16.67$  & $50.0$  & $\textcolor{violet}{\mathbf{100.0}}$   & $95.44$ & $\textcolor{violet}{\mathbf{100.0}}$ & $96.43$  & $97.61$  \\
RTRGN & $\textcolor{violet}{\mathbf{100.0}}$   & $\textcolor{violet}{\mathbf{100.0}}$ & $\textcolor{violet}{\mathbf{100.0}}$ & $\textcolor{violet}{\mathbf{100.0}}$  & $\textcolor{violet}{\mathbf{100.0}}$ & $\textcolor{violet}{\mathbf{100.0}}$   & $\textcolor{violet}{\mathbf{100.0}}$ & $\textcolor{violet}{\mathbf{100.0}}$ & $\textcolor{violet}{\mathbf{100.0}}$  & $\textcolor{violet}{\mathbf{100.0}}$  \\ \hline
 \end{tabular}}
 \label{tab:exp on isomorphism}
\end{table}


The results are shown in Table~\ref{tab:exp on isomorphism}. RTRGN achieves perfect accuracy in both experiments, attributed to its heterogeneous revision (HR). Notably, PINT-pos fails entirely to correctly identify true negatives in the Oscillating CSL task. TGN and GRU-GCN on the other hand exhibit partial failures in accurately distinguishing challenging non-isomorphic graphs. One such failure case is illustrated in Figure~\ref{fig:cun_demo}. 

\subsection{Ablation Studies}

\begin{figure}[t]
    \centering
    \vspace{-3mm}
    \includegraphics[width=140mm]{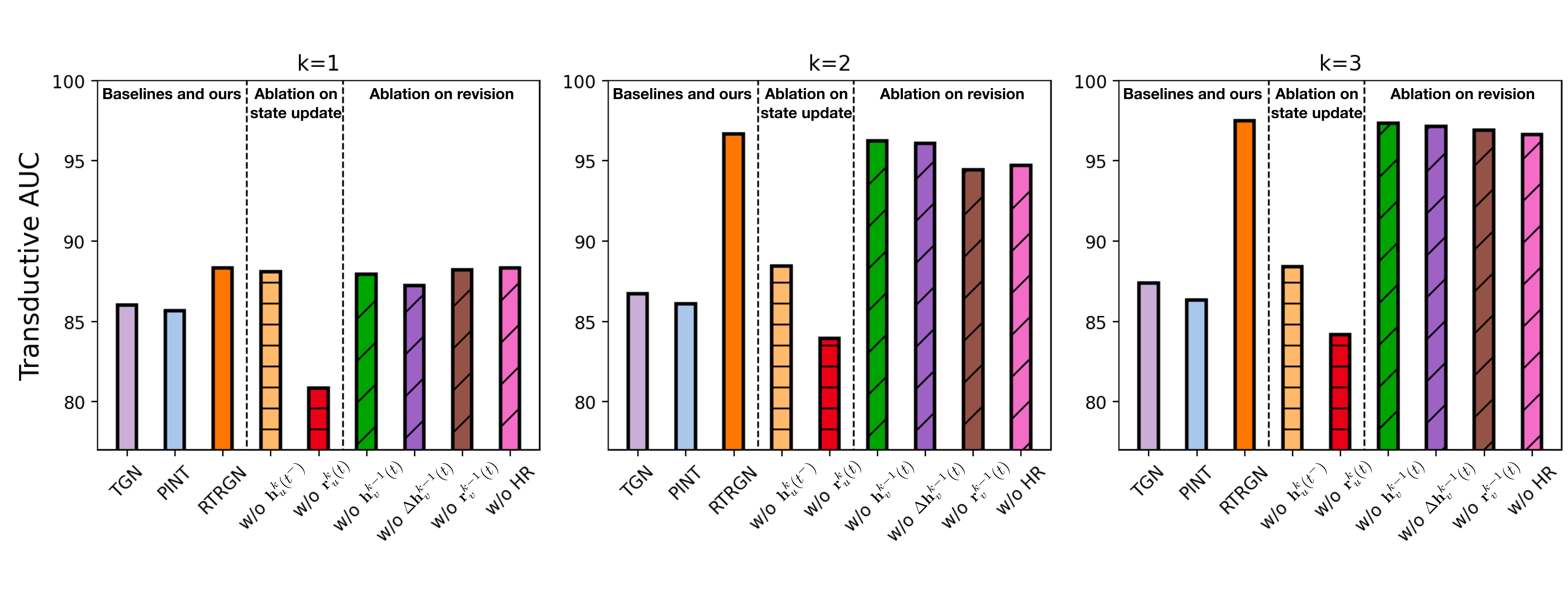}
    \vspace{-10mm}
    \caption{Transductive AUC (\%) on MovieLens dataset. Results averaged over 10 runs.}
    \label{fig:main ablation movie}
\end{figure} 

We conduct a set of ablation studies to dissect the impact of each proposed component in our RTRGN. The experiments are conducted on three representative datasets: MovieLens, Wikipedia, and UCI. The results for MovieLens are presented in Figure~\ref{fig:main ablation movie}, while the results for the remaining two datasets, which validate the consistency of our findings, are included in Appendix~\ref{appendix: Experiment Results}. 



\paragraph{Ablation on state update.} 
Our first ablation study is on $\mathbf{h}_{u}^{k}(t^{-})$, where we set $\mathbf{h}_{u}^{k}(t^{-})=\mathbf{0}$ in \eqref{equ:TRGN_recursion} for $k>0$, hence do not retain the hidden state but compute aggregation dynamically, denoted as w/o $\mathbf{h}_{u}^{k}(t^{-})$. The results suggest that including the hidden state is crucial for further improving performance when $k>1$. As the second ablation study, we set $\mathbf{r}_{u}^{k}(t)=\mathbf{0}$ in \eqref{msg}, which leads to a pure RNN design without neighbor aggregation, denoted as w/o $\mathbf{r}_{u}^{k}(t)$. It results in significant performance drops, indicating the criticality of integrating revisionary neighbor information in the message. 


\paragraph{Ablation on TR module.} 
The rest ablation studies are performed on the temporal revision module. As the first, we remove $\mathbf{h}_{v}^{k-1}(t)$ in \eqref{equ:revision_recursion}, denoted as w/o $\mathbf{h}_{v}^{k-1}(t)$. Interestingly, this setting exhibits only a marginal performance drop and even slightly outperforms the setting denoted as w/o $\Delta\mathbf{h}_{v}^{k-1}(t)$ where $\Delta\mathbf{h}_{v}^{k-1}(t)$ is removed instead. It suggests that the effectiveness of revision can primarily rely on the state change of neighbors, not necessarily the neighbor states. This may be attributed to that the hidden states have already incorporated such information through previous messages. As the last two settings, we remove $\mathbf{r}_{v}^{k-1}(t)$ in \eqref{equ:revision_recursion}, denoted as w/o $\mathbf{r}_{v}^{k-1}(t)$, and use the homogeneous version of revision \eqref{equ:revision_recursion no H} instead of the heterogeneous version \eqref{equ:revision_recursion}, denoted as w/o HR. Both settings result in noticeable performance drops, indicating the benefits of incorporating revisions of neighbors in the TR module, as well as the introduction of heterogeneity in revision. 

More experiments, e.g., analysis of computational time and complexity, ablation studies on $d$, etc., can be found in Appendix~\ref{appendix: Experiment Results}. 

\section{Conclusion}

We have proposed a new aggregation layer, named recurrent temporal revision (RTR) for temporal graphs. We have shown that it leads to superior expressiveness that goes beyond Temporal-1WL. We have also demonstrated that it can lead to significant improvements over current state-of-the-art methods. As an important breakthrough, we have demonstrated that it enables effective utilization of multiple aggregation layers, while existing methods have difficulties in effectively harnessing the power of multiple layers. 





\bibliographystyle{plain} 
\bibliography{neurips_2023}

\newpage
\appendix

\section{Appendix: Model Details and Training Scheme}
\label{Appendix:A}

\subsection{Base Embeddings}
\label{appendix: base case}

The level-$1$ embeddings is built upon the level-$0$ embeddings, which we referred to as the base embeddings. 
We have mainly considered two approaches for the base embeddings $\mathbf{h}^{0}_{u}(t)$: 

(1) The \emph{explicit parameterization} setting, where the base embeddings are treated as node-wise time-dependent free parameters that are initialized randomly and updated through gradient descent. 

(2) The \emph{implicit parameterization} setting, where the base embeddings are generated via a base RNN: 
\begin{equation}
    \mathbf{h}_{u}^{0}(t)=\textsf{UPDATE}^{0}\big(\mathbf{h}_{u}^{0}(t^{-}),\textsf{MSG}^{0}\big(\mathbf{h}_{u}^{0}(t^{-}),\mathbf{h}_{v}^{0}(t^{-}), \Phi(e_{uv,t})\big)\big). 
\label{equ: base RNN}
\end{equation}

This implicit form is adopted for the base embeddings in most RNN-based models, e.g., TGN and PINT. The $\textsf{MSG}^{0}$ is usually implemented as simple concatenation, which is also the case in our implementation. $\mathbf{h}_{u}^{0}(0)$ are typically initialized as zero-vectors. 


In comparison, the implicit form provides more regularization and tends to perform slightly better in our experiments. On the other hand, explicit methods offer faster computation, since it does not require forward or backward passes through $\textsf{UPDATE}^{0}$. 

Particularly, when (dynamic) node feature $x_{u, t}$ is available, we concatenate it to $\mathbf{h}^{k}_{u}(t)$ for all $k$. 

\subsection{Implementation Details}
\label{appendix: time attention as Temporal Revision}

\textbf{The Temporal Revision (TR)}: 
%
To compute $\mathbf{r}_{u}^{k}(t,\mathcal{S})$, we adopt an attention block to aggregate the revision information of the neighbors. 
Specifically, we first concatenate their lower-level revisions $\mathbf{r}_{v_i}^{k-1}(t, \mathcal{S}\cup\{u\})$ (which is computed recursively), 
the neighbourhood states $\mathbf{h}_{v_i}^{k}(t)$, and their state changes $\Delta\mathbf{h}_{v_i}^{k}(t)$. These concatenated features are then fed into a learnable function, denoted as $\textsf{FUNC}_1$. 
To represent events relative relationships,
we combine the event related feature $\Phi(e_{uv_i,t_i})$ with the indicator feature $\mathbbm{1}[v_i\in \mathcal{S}]$ that indicates whether the neighbor is in the specialty set $\mathcal{S}$. They are then passed through another learnable function, denoted as $\textsf{FUNC}_2$. Then, we construct a temporal revision matrix for node $u$ at layer $k$ as follows: 
\[
\begin{split}
    \mathbf{Z}_{u}^{k}=\Big[&
    \textsf{FUNC}_1^{k-1}\big(\mathbf{r}_{v_1}^{k-1}(t, \mathcal{S}\cup\{u\}), \Delta\mathbf{h}_{v_1}^{k-1}(t), \mathbf{h}_{v_1}^{k-1}(t) \big) \;\Big\|\;\textsf{FUNC}_2^{k-1}\big(\Phi(e_{uv_1,t_1}), \mathbbm{1}[v_1\in \mathcal{S}]\big),\;\;
    \ldots,
    \\&
    \textsf{FUNC}_1^{k-1}\big(\mathbf{r}_{v_N}^{k-1}(t, \mathcal{S}\cup\{u\}), \Delta\mathbf{h}_{v_N}^{k-1}(t),
    \mathbf{h}_{v_N}^{k-1}(t) \big) \;\Big\|\;\textsf{FUNC}_2^{k-1}\big(\Phi(e_{uv_N,t_N}), \mathbbm{1}[v_N\in \mathcal{S}]\big)
    \Big]^\top.
\end{split}
\]
We then forward it to three different linear projections to obtain the \textit{query}, \textit{key} and \textit{value} matrices of the standard self-attention mechanism: 
\begin{equation}
    \mathbf{Q}_{u}^{k}=\mathbf{Z}_{u}^{k}\mathbf{W}^{Q},\;\;
    \mathbf{K}_{u}^{k}=\mathbf{Z}_{u}^{k}\mathbf{W}^{K},\;\;
    \mathbf{V}_{u}^{k}=\mathbf{Z}_{u}^{k}\mathbf{W}^{V}.
\end{equation} 
Then the aggregated revision is calculated as:
\begin{equation}
    \mathbf{r}_{u}^{k}(t,\mathcal{S})=\textsf{Softmax}(\frac{\mathbf{Q}_u^k(\mathbf{K}_u^k)^\top}{\sqrt{d_K}}) \mathbf{V}_u^k,
\end{equation}
where $d_K$ is the dimension of $\mathbf{K}$. The above single-head attention can easily be extended to the multi-heads attention. 

Practically, we found that simply using concatenation for $\textsf{FUNC}_1$ and $\textsf{FUNC}_2$ is sufficient to achieve satisfactory performance. 
In our implementation, the dimension of revision $\mathbf{r}_{u}^{k}(t)$ is the same as the dimension of state $\mathbf{h}_{u}^{k}(t)$, which is $172$.

\textbf{The State Transition Model (\textsf{UPDATE}) and State Revising Message (\textsf{RMSG})}:
Next, we combine the revision with the current event related information using general learnable functions $\textsf{FUNC}_3$ and $\textsf{FUNC}_4$, and then feed the resulting message into the state transition model \textsf{UPDATE}, which we implement using a gated recurrent unit (GRU) cell. 
\begin{equation}
    \mathbf{m}^{k}_{u}(t)=\textsf{FUNC}_4^k\Big(\mathbf{h}_{u}^{k-1}(t),\textsf{FUNC}_3^k\big(\mathbf{r}_{u}^{k}(t,\varnothing),\mathbf{h}_{v}^{k-1}(t),\Phi(e_{uv,t})\big)\Big).
\end{equation}
\begin{equation}
    \mathbf{h}^{k}_{u}(t)=\textsf{GRU}^k\big(\mathbf{h}^{k}_{u}(t^-), \mathbf{m}^{k}_{u}(t)\big).
\end{equation}
In practice, we found that simply using concatenation for $\textsf{FUNC}_3$ is sufficient to give satisfactory performance, while using a \textsf{GRU} cell for $\textsf{FUNC}_4$ leads to strong performance. Specifically, using \textsf{GRU} as $\textsf{FUNC}_4$ clearly outperforms the 2-layer \textsf{MLP} in 7 out of the 8 tested link prediction benchmarks. For the \textsf{UPDATE} function, besides \textsf{GRU} cell we have also tried \textsf{RNN} cell and weighted averaging. \textsf{RNN} cell performs similarly to \textsf{GRU} cell, and they are both better than weighted averaging. 




\textbf{The Event Feature Function $\Phi(e_{uv,t'})$}:
We adopt the generic time encoding function $\phi(\cdot)$ to encode the time information, which was first introduced in Time2Vec~\cite{Time2vec} and later used in both TGAT and TGN. We similarly encode the sequence information. We denote the sequence as $s$. 

For event feature function $\Phi(e_{uv,t'})$, we concatenate the time encoding of $\phi_1(t-t')$ and the sequence encoding of $\phi_2(s-s')$. We found the sequence feature useful in a few benchmark datasets. When edge features $e_{uv,t'}$ are available (we abuse the notation of event to represent its feature), we concatenate them behind the time and sequence information. 


\subsection{Training Scheme}
\label{appendix: Training Scheme}

As in normal temporal aggregations~\cite{tgn,pint}, practically, we sample a maximum number of $d$ most recent neighbors for each node to calculate the revision to avoid the unacceptable computational cost. 





During training, we employ an extended version of the classic cache training scheme for temporal graph models \cite{tgn,pint}. For each node, we store the following information: 
(1) its last event $e_{uv,t}$; 
(2) its embeddings before its last event, i.e., $\mathbf{h}_{u}^{k}(t^-)$ for $k\in [0,K]$. These caches enable the GRU updaters ($\textsf{UPDATE}^k$) to perform one-step trace back, which is a standard technique to train RNN models. The revisions, i.e., $\mathbf{r}_{u}^{k}(t,\mathcal{S})$ for $k\in [1,K]$, are computed on-the-fly during training. 


When each new batch comes, we first calculate the one-step-further embeddings for nodes involved in the new batch using the cached embeddings and events, with which we predict the events in the new batch and update parameters via back-propagation. 
For each node $u$ involved in the batch, we update its embedding caches with the computed results $\mathbf{h}_{u}^{k}(t)$ for $k\in[0,K]$. 
In the meantime, we update the stored last events. The list of the most recent neighbors of each involved node will also be updated accordingly.\footnote{
It can be extended to handle multiple events for a node in a batch, e.g., storing the list of events in the last batch for each node, and averaging/summing~\cite{tgn} over  such list the embeddings of interacted nodes and the event features when updating.} 

As the typical implementation for large graphs, the calculation of node embeddings are conducted recursively: we directly request the top-level embeddings of nodes and they will recursively request related lower-level node embeddings until all necessary information is retrieved. 

\subsection{Theoretical Complexity}



The computational complexity of our RTRGN, measured in the number of function evaluations, is the same as classic temporal graph models like TGN and PINT. 
Our revision is the counterpart of the GNN module in baseline methods. The recursion evaluation of node embedding $\mathbf{h}_{u}^{k}(t)$ in both RTRGN and baselines (TGN and PINT) incurs a time complexity of $O(d^k)$ for each single node. In practice, since $k$ is typically a small value (e.g., $3$ in link prediction and $4$ in temporal graph isomorphism test), the computational complexity remains much acceptable.\footnote{When using heterogeneous revision (HR), various specialty set $\mathcal{S}$ can be incurred for a node $v$. However, for a given path, only a single $\mathcal{S}$ incurs for each node in the path. Therefore, the complexity is still $O(d^k)$, considering all paths extending from $u$.} Note that the computation overhead of the \textsf{UPDATE} and \textsf{RMSG} functions is negligible compared to the calculations of the revisions. 

Note that our space requirement is $2$ times more than TGN-like methods. The space requirement for GNN embedding generation in a $k$-layer model is $k|V|$, 
where $|V|$ is the total number of nodes. For RTRGN, we have an additional $k|V|$ for the embedding caches, and $k|V|$ to store the results of revision computation in the intermediate steps. 



\section{Appendix: Theoretical Results on Expressiveness}
\label{Appendix: B}

\subsection{Proof of Proposition \ref{prop:1} and its Corollary \ref{corollary:equal to equal}}
\label{appendix: Proof of Proposition prop0  and its Corollary corollary equal to equal}
\begin{figure}[bt]
    \centering
    \includegraphics[width=100mm]{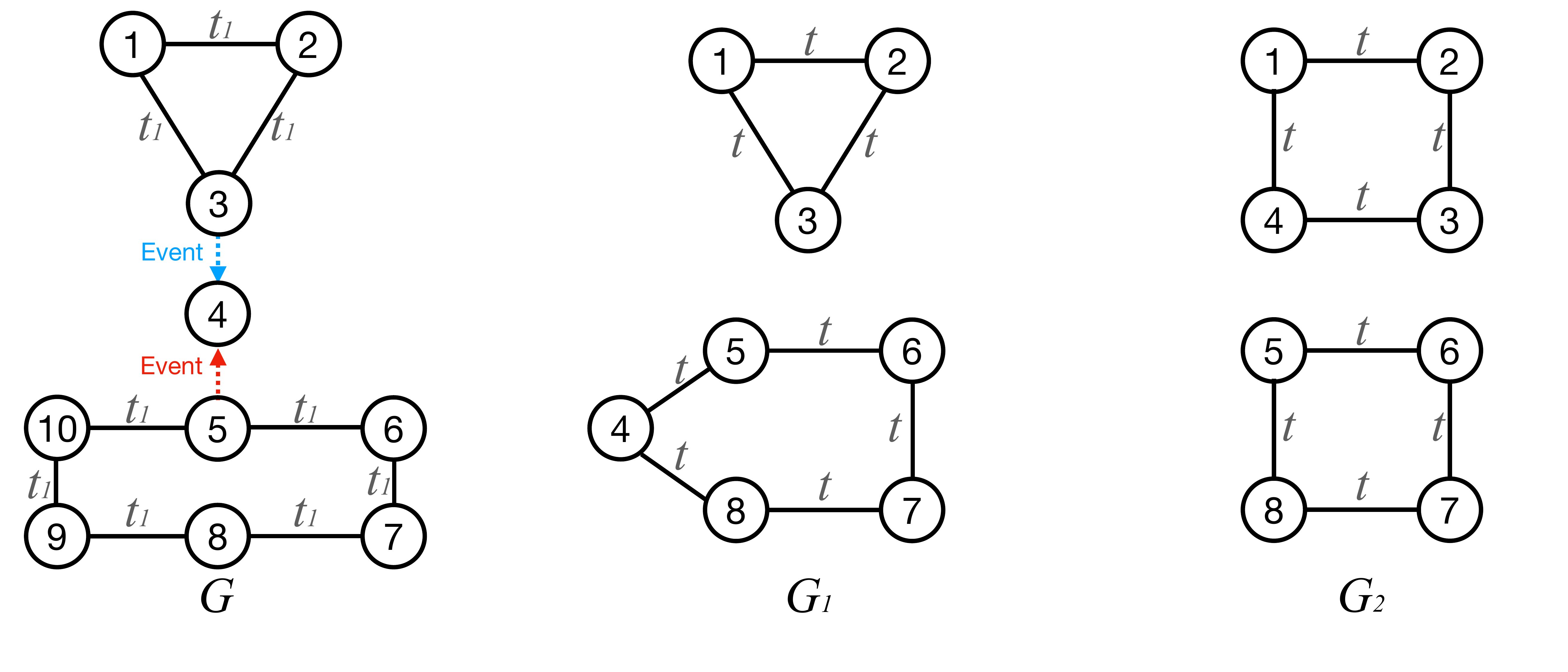}
    \vspace{-3mm}
    \caption{Examples used in the contradiction proof of Proposition~\ref{prop:1}. We assume both node and edge features are zero.}
    \label{fig:contradiction_prop_0}
\end{figure}

\begin{proof}[Proof of the first part.]

Given $f_A \succ_{is} f_B$, we can conclude that there exist two nonisomorphic graphs $G_1(t)$ and $G_2(t)$ such that $f_A$ can distinguish them, whereas $f_B$ cannot. Let's consider each node pair $\langle u,v \rangle \in G_1(t)$. In the case of $f_B$, it fails to produce distinguishable embedding generations for both $u$ and $\pi_B(u)$ (as well as for $v$ and $\pi_B(v)$), where $\pi_B$ represents the bijective mapping of indistinguishable test results for $f_B$ on $G_1(t)$ versus $G_2(t)$. However, there must exist at least one pair $\langle u,v \rangle$ for which $f_A$ can generate distinguishable embedding generations, either for $u$ and $\pi_B(u)$ or for $v$ and $\pi_B(v)$. Otherwise, $f_A$ would fail to distinguish between the two temporal graphs on time $t$ under $\pi_B$, leading to a contradiction with the assumption that $A$ cannot find such bijective mapping $\pi$ to deem $G_1(t)$ versus $G_2(t)$ as indistinguishable. Consequently, $f_A$ can successfully distinguish between the two possible events $e_{uv,t}$ and $e_{\pi_B(u)\pi_B(v),t}$, while $f_B$ cannot.

Thus, we have identified (at least) a specific event pair that $f_A$ can distinguish, whereas $f_B$ fails.

For all event pairs $e_{uv,t}$ and $e_{u'v',t}$ that $f_B$ can distinguish, it achieves this by generating different embedding forms for $u$ and $u'$ (or $v$ and $v'$). Consequently, $f_B$ is capable of identifying the resulting temporal graphs (after the occurrence of the two possible events $e_{uv,t}$ and $e_{u'v',t}$) as non-isomorphic. In this scenario, we can infer that $f_A$ also generates distinct embedding forms for $u$ and $u'$ (or $v$ and $v'$) prior to the occurrence of the two possible events, since $f_A$ also distinguishes the resulting temporal graphs (after the occurrence of the two possible events) as non-isomorphic ($f_A\succ_{is} f_B$).

Therefore, all event pairs that $f_B$ can distinguish, $f_A$ can also distinguish.
As a result, 
\[f_A\succ_{is} f_B\Rightarrow f_A\succ_{lp} f_B.\]
\end{proof}

\begin{proof}[Contradiction proof of the counter part.]
For the converse case, when we have $f_A\succ_{lp} f_B$, we can only demonstrate that there exist some temporal graph pairs that $f_A$ can distinguish while $f_B$ cannot. However, we cannot prove that all graph pairs that $f_B$ can distinguish are also distinguishable by $f_A$.

Firstly, we construct two types of IMP-TGN as follows:
\begin{itemize}
    \item Model $A$: \textit{Cycle of length 6 IMP-TGN} (C6-IMP-TGN), which incorporates a 1-dimensional feature $1$ into the node embeddings if the node is part of a cycle of length 6 (based on the graph's structure and disregarding temporal causality); otherwise, it assigns a value of $0$.
    \item Model $B$: \textit{Longest Path IMP-TGN} (LP-IMP-TGN), which augments node embeddings with a 1-dimensional feature representing the length of the longest (structural) path in the entire graph. 
\end{itemize}

Note that LP-IMP-TGN $\cong_{lp}$ IMP-TGN as the added global feature is irrelevant to predicting the latest event and is indistinguishable for every node.
Furthermore, it can be easily demonstrated that C6-IMP-TGN is superior to LP-IMP-TGN in terms of link prediction ($\succ_{lp}$) using an example from Figure~\ref{fig:contradiction_prop_0}. In this example, in graph $G(t)$, two events are distinguishable by C6-IMP-TGN due to different local features in source nodes 3 and 5, which correspond to a cycle of length 6.

However, in the isomorphism test on $G_1(t)$ and $G_2(t)$ from Figure~\ref{fig:contradiction_prop_0}, C6-IMP-TGN fails to distinguish the graphs (as IMP-TGN itself fails on such a regular graph, and the added feature is $0$ for all nodes). On the other hand, LP-IMP-TGN can produce distinct additional features (the feature output is $5$ for $G_1$ and $4$ for $G_2$), allowing it to distinguish the graphs.
This counterexample demonstrates that $f_A\nsucc_{is} f_B$. Thus, we can conclude that $f_A\succ_{lp} f_B\nRightarrow f_A\succ_{is} f_B$. And in fact, we can only deduce $f_A\succ_{lp} f_B\Rightarrow f_B\nsucc_{is} f_A$.
\end{proof}

\textit{Remark.} C6-IMP-TGN, which is equipped with a node-wise feature value, outperforms LP-IMP-TGN, which only has the global feature, in link prediction. However, C6-IMP-TGN may perform unsatisfactorily in the isomorphism test. This counterexample highlights the difference between the temporal link prediction task and the temporal graph isomorphism test. While link prediction primarily focuses on the local structure, the isomorphism test requires an additional emphasis on global distinguishability.

We can apply a similar proof technique to obtain the following result. 
\begin{corollary}
Being equally expressive in isomorphism test implies being equally expressive in link prediction ($f_A\cong_{is} f_B\Rightarrow f_A\cong_{lp} f_B$).
\label{corollary:equal to equal}
\end{corollary}

\subsection{Proof of Proposition \ref{prop:2}}
\label{appendix: Proof of Proposition prop1}
PINT and Time-then-IMP differ in the design of their base RNNs. The base RNN in PINT is defined by equation \eqref{equ: base RNN}, while the base RNN in Time-then-IMP can be represented as 
\[
\mathbf{h}^0_{u}(t)=\textsf{RNN}_{node}\big(\mathbf{h}^0_{u}(t^-),x_{u,t}\big),\;\;\;\;
\mathbf{e}_{uv}(t)=\textsf{RNN}_{edge}\big(\mathbf{e}_{uv}(t^-),e_{uv,t}\big).
\]
where $x_{u,t}$ and $e_{uv,t}$ represent the node and edge features associated with the respective events.

It can be observed that the edge RNNs in Time-then-IMP play a role in processing and producing edge features, which distinguishes them from the conventional base RNNs that primarily operate on nodes. However, note that in the isomorphism test, the edge features are only used in the generation of node embeddings but no other ways. 


\begin{lemma}
PINT is equally expressive as Temporal-1WL test in temporal graph isomorphism test.
\end{lemma}
This is proved in~\cite{pint}, where they demonstrate the existence of a PINT model that produces the same output as the Temporal-1WL test. The proof establishes the equivalence between PINT and Temporal-1WL in terms of their expressive power.

In our generalized proof, we extend the expressiveness analysis from the static setting in \cite{gin} to the temporal setting for Time-then-IMP. We provide additional insights that considering stacked features, such as $e_{uv,(0:t)}$ instead of the feature at the current timestamp $e_{uv,t}$, does not provide additional expressiveness. 

\begin{proof}
The study by \cite{siegelmann1992computational} demonstrates that a recurrent neural network (RNN) with a sufficient number of hidden neurons can serve as a highly expressive sequence model, approximating a universal Turing machine. Given this insight, we can view the representations of node and edge attribute sequences in an RNN as a memory copy of all historical inputs, with no loss of information. Specifically, we can represent the node embedding as $\mathbf{h}^0_{u}(t)=[x_{u,(0:t)}]$, and the edge embedding as $\mathbf{e}_{uv}(t)=[e_{uv,(0:t)}]$. In these representations, $[x_{v,(0:t)}]$ and $[e_{uv,(0:t)}]$ are matrices or tensors composing of vectors, where each vector is the concatenation of all given node or edge features across all timestamps up to time $t$. If no features are provided, these vectors can be set to zero.

As a result, the embedding generation can be seen as an injective message passing (IMP) operation applied to the stacked graphs, where the stacked graphs consist of the historical node and edge attributes concatenated along the time dimension.
\begin{equation}
\begin{split}
    &\mathbf{h}^k_{u}(t)=\textsf{GNN}^{k}_{u}\big(\text{Node Attributes}=[x_{v,(0:t)}]_{v\in V(t)},\;\text{Edge Attributes}=[e_{uv,(0:t)}]_{uv\in E(t)}\big).
\end{split}
\end{equation}

Here $\textsf{GNN}^{k}_{u}$ is the summary of the embedding generation process following the temporal neighbor aggregation and update, i.e. \eqref{equ:agg} and \eqref{equ:agg_update}, by setting the base input $\mathbf{h}^0_{u}(t)$ as node attributes, and concatenating the edge attributes into the aggregation \eqref{equ:agg}.
Note that any pre-computed features before passing the raw features to GNN can be absorbed by the first injective message passing layer in GNN, so the above form is valid when including any pre-computed features from raw features.

Suppose that up to time $T$, the Temporal-1WL algorithm is unable to distinguish the temporal graphs $G_1(T)$ and $G_2(T)$. This implies that for any iteration $k$ of the Temporal-1WL test conducted at time $t\leq T$, the graphs $G_1(t)$ and $G_2(t)$ consistently exhibit the same multiset of node labels for each node $u$ ($\in G_1$ and its counterpart $u'\in G_2$ in the bijective mapping), represented by the coloring ${c^k(u)}$ versus ${c^k(u')}$. Furthermore, the two graphs also share the same multiset of node neighbors $\{\!\!\{(c^k(v),\Phi(e_{uv,t')})|e_{uv,t'}\in E_1(t)\}\!\!\}$, where $E_1(t)$ represents the set of edges in $G_1(t)$ versus that of $G_2(t)$, and the double brackets represent the multiset.

Again let $\Phi$ be a generic encoding of event feature, and time information, etc. We then can observe that when examining the Temporal-1WL test conducted at time $t$, the fact that $c^{k+1}(u)=c^{k+1}(u')$ for all iterations $k$ implies that the following inputs to the $k+1$ iteration of test are equivalent:
\begin{equation}
\big(c^k(u),\big\{\!\!\big\{(c^k(v),\Phi(e_{uv,t'}))|e_{uv,t'}\in E_1(t)\big\}\!\!\big\}\big)
=
\big(c^k(u'),\big\{\!\!\big\{(c^k(v'),\Phi(e_{u'v',t'}))|e_{u'v',t'}\in E_2(t)\big\}\!\!\big\}\big)
\label{equ:multiset color}
\end{equation}
Then it follows that the multisets considering the stacked features is also equivalent:
\begin{equation}
\big(c^k(u),\big\{\!\!\big\{(c^k(v),\Phi(e_{uv,(0:t')}))|e_{uv,t'}\in E_1(t)\big\}\!\!\big\}\big)
=
\big(c^k(u'),\big\{\!\!\big\{(c^k(v'),\Phi(e_{u'v',(0:t')}))|e_{u'v',t'}\in E_2(t)\big\}\!\!\big\}\big)
\label{equ:multiset color all t}
\end{equation}
Indeed, by splitting the multisets into two parts, one for neighbors that only observe interactions before $t$ and another for neighbors that observe interactions exactly at time $t$, one can demonstrate the equivalence of these two multisets for node $u$ and node $u'$. This can be proven by induction, ensuring that the assumption that Temporal-1WL conducted prior to time $t$ deems the graphs isomorphic holds true without contradiction.

Assuming that at time $t$, for the same 1WL node label $c^k(u) = c^k(u')$ at iteration $k$, we also have the output of Time-then-IMP as identical $\mathbf{h}^k_{u} = \mathbf{h}^k_{u'}$ for any node $u\in G_1(t)$ (and its corresponding image $u'\in G_2(t)$ in the bijective mapping given by the Temporal-1WL). It is worth noting that this assumption holds true for the base case $k = 0$ since Temporal-1WL and Time-then-IMP both start with the same node features.

Based on this assumption, it follows that the inputs to the  Time-then-IMP's $k+1$ iteration
\begin{equation}
\big(\mathbf{h}^k_u,\big\{\!\!\big\{(\mathbf{h}^k_v,\Phi(e_{uv,(0:t')}))|e_{uv,t'}\in E_1(t)\big\}\!\!\big\}\big)
=
\big(\mathbf{h}^k_{u'},\big\{\!\!\big\{(\mathbf{h}^k_{v'},\Phi(e_{u'v',(0:t')}))|e_{u'v',t'}\in E_2(t)\big\}\!\!\big\}\big)
\label{equ:multiset h}
\end{equation}
are also the same. 
When the same injective aggregation and update functions are applied, the output will likewise be the same. Consequently, Time-then-IMP will also classify the two graphs as isomorphic on the $k+1$ iteration of its test (or equivalently, the $(k+1)$-layer output of the Time-then-IMP model) on time $t$.

By applying induction, we can conclude that if $k$ iterations of the Temporal-1WL test result in the graphs being indistinguishable, the output of the most expressive Time-then-IMP model will also be indistinguishable.

On the other hand, if at iteration $k+1$ of the Temporal-1WL test, the result becomes distinguishable, i.e., $c^{k+1}(u) \neq c^{k+1}(u')$, it implies that the inputs mentioned in equation \eqref{equ:multiset color} are different. Consequently, the inputs to the $k+1$ iteration of Time-then-IMP in equation \eqref{equ:multiset h} will also be different. Applying injective aggregation and update functions, we can deduce that $\mathbf{h}^{k+1}_{u} \neq \mathbf{h}^{k+1}_{u'}$, meaning that Time-then-IMP's predictions also become distinguishable at iteration $k+1$. 

In conclusion, we have identified and demonstrated that Time-then-IMP will produce the same output as Temporal-1WL in the temporal graph isomorphism test. As a result, Time-then-IMP $\cong_{is}$ Temporal-1WL. Combines with the above lemma, we
\[
\text{Temporal-1WL}\; \cong_{is}\; \text{Time-then-IMP}\; \cong_{is}\; \text{PINT}.
\]
With Corollary \ref{corollary:equal to equal}, we also have 
\[
\text{Temporal-1WL}\; \cong_{lp}\; \text{Time-then-IMP}\; \cong_{lp}\; \text{PINT}.
\]
\end{proof}

\subsection{Proof of Proposition \ref{prop:3}}
\begin{proof}
We begin by demonstrating that any arbitrary IMP-TGN representation (such as PINT's) can be replicated by an RTRGN representation that produces the same embeddings for identical temporal graph inputs. We first examine RTRGN without heterogeneous revision, assuming that the indicator function in equation \eqref{equ:revision_recursion} is always set to zero accordingly. By further setting a portion of the inputs in the revision step to zero, we can restore the original IMP-TGN representation.
\[
\mathbf{r}_{u}^{k}(t) = \textsf{TR}^{k}\Big(\Big\{\!\!\Big\{\big(\mathbf{h}_{v}^{k-1}(t),\Phi(e_{uv,t'})\big) \;\big\mid\; e_{uv,t'}\in E(t^-)\Big\}\!\!\Big\}\Big), 
\]
Let the injective function $\textsf{TR}^{k}$ have the maximum expressiveness (i.e., the universal function approximator) such that it simply copies the inputs. Then, we can reformulate \eqref{msg} into
\[
\begin{aligned}
\mathbf{m}_{u}^{k}(t) = &\textsf{RMSG}^{k}\big(
\mathbf{h}_{u}^{k-1}(t),\mathbf{h}_{v}^{k-1}(t),
\mathbf{r}_{u}^{k}(t),
\Phi(e_{uv,t})\big)
\\=&\textsf{F}_1^{k}\big(
\mathbf{h}_{u}^{k-1}(t), \textsf{F}_2^{k}\big(\mathbf{h}_{v}^{k-1}(t),
\mathbf{r}_{u}^{k}(t),
\Phi(e_{uv,t})\big)\big)
\\=&\textsf{F}_1^{k}\Big(
\mathbf{h}_{u}^{k-1}(t), \textsf{F}_2^{k}\Big((\mathbf{h}_{v}^{k-1}(t),\Phi(e_{uv,t})),\Big\{\!\!\Big\{\big(\mathbf{h}_{v'}^{k-1}(t),\Phi(e_{uv',t'})\big) \;\big\mid\; e_{uv',t'}\in E(t^-)\Big\}\!\!\Big\}
\Big)\Big)
\\=&\textsf{F}_1^{k}\Big(
\mathbf{h}_{u}^{k-1}(t), \textsf{F}_2^{k}\Big(\Big\{\!\!\Big\{\big(\mathbf{h}_{v}^{k-1}(t),\Phi(e_{uv,t'})\big) \;\big\mid\; e_{uv,t'}\in E(t)\Big\}\!\!\Big\}\Big)\Big)
,
\end{aligned}
\]
Note that we adopt the same base RNN as IMP-TGN (illustrated on Appendix~\ref{Appendix  :A}).
Furthermore, by specifically setting $\textsf{F}_1^{k}=\textsf{COMBINE}^{k}$, $\textsf{F}_2^{k}=\textsf{AGG}^{k}$, and $\textsf{UPDATE}^{k}(\mathbf{h},\mathbf{m})=\mathbf{m}$, we effectively restore the embedding generation process of IMP-TGN.

Thus so far we have proved RTRGN $\succeq_{is}$ IMP-TGN.

Next, we employ the Oscillating CSL task (Figure~\ref{fig:OscillatingCSL}) to demonstrate that our RTRGN exhibits greater expressiveness compared to the Temporal-1WL class, which encompasses three models: Temporal-1WL, IMP-TGN, and Time-then-IMP, as discussed earlier.

Firstly, the methods in the Temporal-1WL class are unable to distinguish between the two temporal graphs depicted in Figure~\ref{fig:OscillatingCSL}. We denote the top graph as $G_1$ and the bottom graph as $G_2$. When referring to nodes in $G_2$, we label them with a prime symbol, such as $u$ for $G_1$ and $u'$ for $G_2$.
It is important to note that we have assumed the graphs are unattributed, meaning that the features associated with each event are always zero: $x_{u,t}=\mathbf{0}$ and $e_{uv,t}=\mathbf{0}$. In this sense, the generic feature encoding simply records the time (or sequence) information: $\Phi(e_{uv,t'})=t'$.

At each time $t_n$ (assuming without loss of generality that $n$ is an odd number), let's consider the scenario where $c^k(u)=c^k(v)=c^k(u')=c^k(v')=c^k$ for any nodes $u,v,u',v'$ in both graphs (it can be easily verified that the base case $k=0$ holds for any $t$ since the graphs are unattributed). In this case, we can show that the $k+1$ iteration, corresponding to the embedding generated by the $k+1$ Temporal-1WL models, will also be identical. This can be demonstrated by the inputs to the $k+1$ iteration:
\[
\begin{aligned}
&\big(c^k(u),\big\{\!\!\big\{(c^k(v),\Phi(e_{uv,t'}))|e_{uv,t'}\in E_1(t_n)\big\}\!\!\big\}\big)
\\=&
\big(c^k,\big\{\!\!\big\{(c^k,t')|e_{uv,t'}\in E_1(t_n)\big\}\!\!\big\}\big)
\\=&
\big(c^k,\big\{\!\!\big\{(c^k,t')|t'\in \{\!\!\{t_1,t_2,\ldots t_n\}\!\!\}\cup\{\!\!\{t_1,t_3,\ldots,t_n\}\!\!\}\cup\{\!\!\{t_2,t_4,\ldots,t_{n-1}\}\!\!\}
\\&\;\;\;\;\;\;\;\;\;\;\cup\{\!\!\{t_2,t_4,\ldots,t_{n-1}\}\!\!\}\cup\{\!\!\{t_1,t_3,\ldots,t_n\}\!\!\}\cup\{\!\!\{t_1,t_2,\ldots t_n\}\!\!\}
\big\}\!\!\big\}\big)
\\=&
\big(c^k(u'),\big\{\!\!\big\{(c^k(v'),\Phi(e_{u'v',t'}))|e_{u'v',t'}\in E_2(t_n)\big\}\!\!\big\}\big).
\end{aligned}
\]
As a result, Temporal-1WL cannot distinguish between these two temporal graphs because the generated embeddings for any node $u\in G_1$ and $u'\in G_2$ are the same for any $k$.

Next, we demonstrate that our RTRGN model can successfully distinguish between the two temporal graphs depicted in Figure~\ref{fig:OscillatingCSL}.

Consider the time $t_n$ (without loss of generality, assuming $n$ is an odd number) that is sufficiently large. For simplicity, we remove the $\Delta\mathbf{h}$ term in \eqref{equ:revision_recursion}, as without this term the expressiveness of RTRGN can already surpass Temporal-1WL class.
First, let's focus on $G_1(t_n)$. For example, we observe the following computation:
\[
\begin{aligned}
\mathbf{r}_{2}^{k}(t_n,\{1,7\})=\!\textsf{TR}^{k}\Big(
& \Big\{\!\!\Big\{\big(\mathbf{r}_{1}^{k-1}(t_n,\{1,2,7\}),\mathbf{h}_{1}^{k-1}(t_n),\Phi(e_{21,t'}),\mathbbm{1}\big)\big\mid t'\in\{\!\!\{t_1,t_2,\ldots t_{n-1}\}\!\!\}\Big\}\!\!\Big\}
\\\cup&  \Big\{\!\!\Big\{\big(\mathbf{r}_{3}^{k-1}(t_n,\{1,2,7\}),\mathbf{h}_{3}^{k-1}(t_n),\Phi(e_{23,t'}),\mathbb{0}\big)\big\mid t'\in\{\!\!\{t_1,t_2,\ldots t_{n-1}\}\!\!\}\Big\}\!\!\Big\}
\\\cup&  \Big\{\!\!\Big\{\big(\mathbf{r}_{4}^{k-1}(t_n,\{1,2,7\}),\mathbf{h}_{4}^{k-1}(t_n),\Phi(e_{24,t'}),\mathbb{0}\big)\big\mid t'\in\{\!\!\{t_1,t_3,\ldots t_{n-1}\}\!\!\}\Big\}\!\!\Big\}
\\\cup&  \Big\{\!\!\Big\{\big(\mathbf{r}_{5}^{k-1}(t_n,\{1,2,7\}),\mathbf{h}_{5}^{k-1}(t_n),\Phi(e_{25,t'}),\mathbb{0}\big)\big\mid t'\in\{\!\!\{t_2,t_4,\ldots t_{n-2}\}\!\!\}\Big\}\!\!\Big\}
\\\cup& \Big\{\!\!\Big\{\big(\mathbf{r}_{6}^{k-1}(t_n,\{1,2,7\}),\mathbf{h}_{6}^{k-1}(t_n),\Phi(e_{26,t'}),\mathbb{0}\big)\big\mid t'\in\{\!\!\{t_2,t_4,\ldots t_{n-2}\}\!\!\}\Big\}\!\!\Big\}
\\\cup& \Big\{\!\!\Big\{\big(\mathbf{r}_{7}^{k-1}(t_n,\{1,2,7\}),\mathbf{h}_{7}^{k-1}(t_n),\Phi(e_{27,t'}),\mathbbm{1}\big)\big\mid t'\in\{\!\!\{t_1,t_3,\ldots t_{n-1}\}\!\!\}\Big\}\!\!\Big\}
\Big).
\end{aligned}
\]
In the above calculation, we can show that $\mathbf{h}_{u}^{k-1}=\mathbf{h}^{k-1}$ for any node $u$ and iteration $k$, which arises from the symmetry in the temporal graph. Additionally, note that $\Phi(e_{uv,t'})=t'$ because the graph is unattributed.
The result of the above calculation differs from that of $G_2(t_n)$, as shown below (where we neglect the prime symbol for simplicity):
\[
\begin{aligned}
\mathbf{r}_{2}^{k}(t_n,\{1,7\})=\!\textsf{TR}^{k}\Big(
& \Big\{\!\!\Big\{\big(\mathbf{r}_{1}^{k-1}(t_n,\{1,2,7\}),\mathbf{h}_{1}^{k-1}(t_n),\Phi(e_{21,t'}),\mathbbm{1}\big)\big\mid t'\in\{\!\!\{t_1,t_2,\ldots t_{n-1}\}\!\!\}\Big\}\!\!\Big\}
\\\cup&  \Big\{\!\!\Big\{\big(\mathbf{r}_{3}^{k-1}(t_n,\{1,2,7\}),\mathbf{h}_{3}^{k-1}(t_n),\Phi(e_{23,t'}),\mathbb{0}\big)\big\mid t'\in\{\!\!\{t_1,t_2,\ldots t_{n-1}\}\!\!\}\Big\}\!\!\Big\}
\\\cup&  \Big\{\!\!\Big\{\big(\mathbf{r}_{4}^{k-1}(t_n,\{1,2,7\}),\mathbf{h}_{4}^{k-1}(t_n),\Phi(e_{24,t'}),\mathbb{0}\big)\big\mid t'\in\{\!\!\{t_2,t_4,\ldots t_{n-2}\}\!\!\}\Big\}\!\!\Big\}
\\\cup&  \Big\{\!\!\Big\{\big(\mathbf{r}_{5}^{k-1}(t_n,\{1,2,7\}),\mathbf{h}_{5}^{k-1}(t_n),\Phi(e_{25,t'}),\mathbb{0}\big)\big\mid t'\in\{\!\!\{t_1,t_3,\ldots t_{n-1}\}\!\!\}\Big\}\!\!\Big\}
\\\cup& \Big\{\!\!\Big\{\big(\mathbf{r}_{6}^{k-1}(t_n,\{1,2,7\}),\mathbf{h}_{6}^{k-1}(t_n),\Phi(e_{26,t'}),\mathbb{0}\big)\big\mid t'\in\{\!\!\{t_1,t_3,\ldots t_{n-1}\}\!\!\}\Big\}\!\!\Big\}
\\\cup& \Big\{\!\!\Big\{\big(\mathbf{r}_{7}^{k-1}(t_n,\{1,2,7\}),\mathbf{h}_{7}^{k-1}(t_n),\Phi(e_{27,t'}),\blue{\mathbbm{1}}\big)\big\mid t'\in\blue{\{\!\!\{t_2,t_4,\ldots t_{n-2}\}\!\!\}}\Big\}\!\!\Big\}
\Big)
\end{aligned}
\]
The difference between the two expressions arises essentially from the presence of the heterogeneous feature (as the indicator) and persists even if $\mathbf{r}_{u}^{k-1}(t_n,\{1,2,7\})$ are the same for any $u$. Without such a heterogeneous feature, it is easy to show that the above two expressions for $G_1$ and $G_2$ would yield the same result for $\mathbf{r}_{2}^{k}(t_n,\{1,7\})$, if $\mathbf{r}_{u}^{k-1}(t_n,\{1,2,7\})$ is the same for all nodes $u$.

Similarly, one can demonstrate that the calculation of $\mathbf{r}_{2}^{k}(t_n,\{1,7\})$ in $G_1$ is different from any $\mathbf{r}_{u}^{k}(t_n,\{1,7\})$ in $G_2$ for all nodes $u$. Consequently, computations relying on this term, such as $\mathbf{r}_{7}^{k+1}(t_n,\{1\})$, and subsequent computations, such as $\mathbf{r}_{1}^{k+2}(t_n,\varnothing)$, will also yield different results. Consequently, with injective message functions and update functions, the resulting $\mathbf{h}_{1}^{k+2}(t_n)$ will differ between $G_1$ and $G_2$. This implies (a three or more layer, $k\geq3$) RTRGN can distinguish the two graphs $G_1$ and $G_2$ on time $t_n$.

Therefore, we have shown that RTRGN is more expressive than Temporal-1WL in terms of the isomorphism test: RTRGN $\succ_{is}$ Temporal-1WL. With Corollary \ref{corollary:equal to equal}, we also have RTRGN $\succ_{lp}$ Temporal-1WL. These results can be generalized to other methods in the Temporal-1WL class, utilizing the results in Proposition \ref{prop:2}.

\end{proof}

\subsection{Proof of Proposition~\ref{prop:4}}
\label{appendix:proof of prop4}
\begin{proof}[Proof of false positive in temporal graph isomorphism test]
In the Oscillating CSL classification task (Figure \ref{fig:OscillatingCSL}), we consider two temporal graphs on $t=t_1$: $G_1(t)$ follows a $\mathbf{C}_{7,2}$ CSL versus $G_2(t)$ follows a $\mathbf{C}_{7,3}$ CSL. We assume that both node and edge attributes are zero. It is worth noting that these two temporal graphs are isomorphic because the underlying static CSL structure, $\mathbf{C}_{7,2}$ and $\mathbf{C}_{7,3}$ are isomorphic.

To compute the positional features, we label the nodes in a clockwise manner as shown in Figure~\ref{fig:OscillatingCSL}. Then, we can compute the 1-dimensional positional features for each node at time $t_1$. These positional features correspond to the number of length-1 walks between two nodes and can be represented as row vectors below, using the adjacency matrix at time $t_1$:
\[
\text{Pos}(G_1) = 
\begin{pmatrix}
0 & 1 & 1 & 0 & 0 & 1 & 1 \\
1 & 0 & 1 & 1 & 0 & 0 & 1 \\
1 & 1 & 0 & 1 & 1 & 0 & 0 \\
0 & 1 & 1 & 0 & 1 & 1 & 0 \\
0 & 0 & 1 & 1 & 0 & 1 & 1 \\
1 & 0 & 0 & 1 & 1 & 0 & 1 \\
1 & 1 & 0 & 0 & 1 & 1 & 0
\end{pmatrix},\;\;
\text{Pos}(G_2) = 
\begin{pmatrix}
0 & 1 & 0 & 1 & 1 & 0 & 1 \\
1 & 0 & 1 & 0 & 1 & 1 & 0 \\
0 & 1 & 0 & 1 & 0 & 1 & 1 \\
1 & 0 & 1 & 0 & 1 & 0 & 1 \\
1 & 1 & 0 & 1 & 0 & 1 & 0 \\
0 & 1 & 1 & 0 & 1 & 0 & 1 \\
1 & 0 & 1 & 1 & 0 & 1 & 0
\end{pmatrix}.
\]
We can verify that there is no permutation of rows, denoted by $\pi$, that satisfies $\pi \cdot \text{Pos}(G_1) = \text{Pos}(G_2)$. This indicates that, regardless of the permutation, the node aggregation will be distinguished by the injective function applied to the distinct positional features. Consequently, this will result in different embedding generations and, therefore, output nonisomorphic results for the isomorphic temporal graph pairs.

The property of a feature or function is sensitive to the initial labeling of nodes is known as ``\textit{permutation sensitivity}'' \cite{relationalpooling, grugcn}. The reason that PINT-pos outputs false positives is primarily due to its positional feature being permutation sensitive. On the other hand, it can be easily shown that the other models we discussed have \textit{permutation-invariant} features and embedding generations \cite{relationalpooling, grugcn}. Hence, we can conclude that these models will not produce false positive results.
\end{proof}

When considering two possible events $e_{u_1v_1,t}$ and $e_{u_2v_2,t}$ on the same temporal graph $G$, if it is the case that after either event occurs, the resulting two temporal graphs are isomorphic, we consider these two events to be \textit{inherently indistinguishable}. By imposing the constraint that the resulting temporal graphs must be nonisomorphic, we provide a means to evaluate the ground truth: If adding either event still results in isomorphic temporal graphs, we consider it illogical to distinguish between these two events.

\begin{proof}[Proof of false positive in temporal link prediction]
Similarly, false positives can arise in temporal link prediction. Consider the temporal graph $G(t)$ at $t=t_1$, which follows a $\mathbf{C}_{7,2}$ CSL (top graph on Figure~\ref{fig:OscillatingCSL}). With positional features, each node's state will be distinct as they are generated by applying injective functions to their respective positional features, $\text{Pos}(G_1)$, which is different for all nodes. This allows PINT-pos to distinguish between any pairs of nodes in temporal link prediction. However, we know that pairs of nodes such as $\langle 1,2\rangle$ and $\langle 2,3\rangle$ are inherently indistinguishable, as adding these two links separately to the graph $G(t)$ results in two isomorphic graphs.

The reason PINT-pos produces false positives in temporal link prediction is fundamentally the same as in the isomorphism test: its positional feature is permutation sensitive.
\end{proof}


\subsection{A Model with Maximal Expressiveness}
\label{appendix:A Model with Maximal Expressiveness}
\begin{lemma}
Let the stacking form of a temporal graph $G(t)$ be defined as a static graph with the following attributes:
\[
\text{Node Attributes}=[x_{v,(0:t)}]_{v\in V(t)},\;\text{Edge Attributes}=[e_{uv,(0:t)}]_{uv\in E(t)}.
\]
Here $[x_{v,(0:t)}]$ and $[e_{uv,(0:t)}]$ are matrices or tensors composing of vectors. Each vector is obtained by concatenating the node or edge features across all timestamps from 0 to $t$. If no such features are provided, these vectors can be set to zero.
Then two temporal graphs are isomorphic, if and only if their two stacking forms are isomorphic.
\label{lemma2}
\end{lemma}
\begin{proof}
Since the temporal graph corresponds to a finite collection of events, the node and edge attributes in the stacking form also form a finite set. This reveals our assumption that the attributes forming finite sets will not be altered when changing the representation to the stacking form.

Suppose two stacking forms of $G_1(t)$ and $G_2(t)$ are isomorphic. It implies the existence of a bijective mapping $\pi$ such that every node $u\in G_1(t)$, $\pi(u)\in G_2(t)$ has exactly the same node attributes $x_{u,t'}$ versus $e_{\pi(u),t'}$ and exactly the same edge attributes $e_{uv,t'}$ versus $e_{\pi(u)\pi(v),t'}$ for every time $t'\leq t$. This equivalence means that $G_1(t)$ and $G_2(t)$ are isomorphic as temporal graphs.
\end{proof}

\cite{siegelmann1992computational} shows that with enough hidden neurons, an RNN can be a most-expressive sequence model
(universal Turing machine approximator). Therefore, if we consider the representations of node and edge attribute sequences generated by such a most-expressive RNN, they can be seen as a memory copy of all the historical inputs, preserving all the information without loss. 

In the context of a time-independent model applied to these RNN outputs, we can equivalently view it as being applied to the stacking form of the temporal graph, as stated in Lemma \ref{lemma2}. This is because the RNN outputs effectively capture the temporal dependencies and encode the relevant information from the entire history, allowing the subsequent model to operate on the stacking form of the temporal graph without losing any information.

Denote the total number of nodes as $|V|$.
Let the augmented edge attribute tensor $\mathbf{A}\in{R}^{|V|\times|V|\times(1+d_e)}$ be a tensor  that combines the aggregated adjacency matrix of $G(t)$ (representing structural connections) with its edge attributes $e_{uv,(0:t)}\in \mathbb{R}^{d_e}$, for an arbitrary permutation of $V$. Additionally,  consider the node attribute matrix $\mathbf{X}=[x_{v,(0:t)}]_{v\in V}$ ($\in{R}^{|V|\times d_v}$) that consolidates all node attributes for the same arbitrary permutation of $V$. The \emph{joint relational pooling} model is defined as
\begin{equation}
    \mathbf{h}_{u}(t)=\bigg(\frac{1}{|V|!}\sum_{\pi\in\Pi({|V|})}\Vec{f}(\mathbf{A}_{\pi,\pi},\mathbf{X}_{\pi})\bigg)_{u}.
\end{equation}
where $\Pi({|V|})$ is the set of all distinct permutations of $V$, and $\Vec{f}$ represents the application of an arbitrary (possibly permutation-sensitive) injective function $f$ over each row of the input ($\in \mathbb{R}^{|V|\times(1+d_e)+d_v}$). When the permutation is applied to $\mathbf{A}$ and $\mathbf{X}$, resulting in $\mathbf{A}_{\pi,\pi},\mathbf{X}_{\pi}$, it permutes the first two dimensions of $\mathbf{A}$ and the first dimension (row) of $\mathbf{X}$. We refer to this model as \textit{RNN with joint Relational Pooling} (RNN-RP).
We can then establish the following proposition, analogous to Theorem 2.1 in~\cite{relationalpooling}:
\begin{proposition}
The RNN-RP model defined in the above equation is the most expressive model in terms of temporal graph isomorphism test and link prediction. With a sufficiently expressive injective function $f$ (i.e., a universal approximator), the RNN-RP model has the ability to output the groundtruth for both tasks.
\label{prop:max expressive}
\end{proposition}
\begin{proof}
Based on Theorem 2.1 in~\cite{relationalpooling}, we know that joint Relational Pooling (RP) is the most expressive model in static graph isomorphism tests and can always output the groundtruth. In our setting, the RNN-RP model, which incorporates joint RP operating on the stacking form of temporal graphs, will also be the most expressive model in temporal graph isomorphism tests, consistently outputting the groundtruth.

Based on Proposition~\ref{prop:1} and Corollary~\ref{corollary:equal to equal}, we can conclude that the RNN-RP model will also be more expressive than any other model in the context of link prediction.
\end{proof}
It is important to note that the embedding generation of the RNN-RP model is permutation-invariant due to its enumeration of all possible permutations and summation of the results. However, it should be acknowledged that RNN-RP may not be practical in real-world scenarios, unless specific symmetries exist, as mentioned in \cite{relationalpooling}. The computational complexity of RNN-RP is intractable, making it challenging to apply in practice.
Nevertheless, RNN-RP serves as a conceptual model that provides an upper bound for the complexity of the most expressive models. Further research and studies may focus on reducing this upper bound of complexity for developing more practical and efficient models while still maintaining high expressiveness.

\newpage

\newpage
\section{Appendix: More Details of Experiment Settings}
\label{Appendix: Details of Experiment Settings}

\subsection{Experiment Settings and Baselines Configuration}

We preprocess all datasets following~\cite{tgn,pint}.  
For dataset preparation, we employ a chronological split of the data, dividing it into training, validation, and test sets in a ratio of $70\%/15\%/15\%$. Additionally, we randomly select $10\%$ nodes and mask them out from the training and validation sets. All models are trained and validated using the masked training and validation sets. Transductive evaluation is performed using the masked test set, while inductive evaluation involves the $10\%$ mask-outed nodes. 
It is important to note that this approach differs slightly from CAWN and NAT, as they retrain the model for the transductive setting without masking any node.
Regarding negative edge sampling for inductive evaluation, we randomly sample negative target nodes from the $10\%$ mask-outed nodes, rather than randomly sampling nodes from the entire dataset as done in CAWN and NAT. 
The choice is motivated by the goal of preventing the model from learning unintended shortcuts to distinguish between masked nodes (newly initialized) and unmasked nodes (well-trained), as such shortcuts would result in a test metric that lacks meaningful interpretation.
During validation and testing, we allowed the update of memories on the fly but did not permit model training, which involves updating the model parameters.

We utilize existing baseline implementations obtained online and ensure fair comparison by calibrating them with the same configuration. In all models, we set the dimensions of the hidden state, memory, and time encoding to 172. Similarly, our model also has a dimension of 172 for revision. For methods involving neighborhood sampling, we used a default value of $d=10$ for the neighborhood size, unless stated otherwise. For models employing multi-head attention, we set the number of heads to 2. 

In our experiments, we conducted each experiment using a Tesla V100 GPU with 16 GB of memory. However, we encountered an error when using the GPU with the CAWN implementation available online (\url{https://github.com/snap-stanford/CAW}). Therefore, we performed the experiments of CAWN using the CPU instead. 

We did not compare our results to PINT-pos due to its extremely high computational complexity and space requirements. The public implementation of PINT-pos requires pre-computation of all positional features, which results in a space requirement proportional to $N_B*|V|*|V|*d_f$ where $|V|$ is the total number of nodes, $N_B$ is the total number of batches, $d_f$ is the positional feature dimension. Even for the smallest dataset UCI, the positional feature generation alone took more than 6 hours. Furthermore, for the larger datasets, even setting a very small positional feature dimension $d_f=4$ would lead to out of memory (OOM) errors on our GPU. Hence, it was not feasible to include PINT-pos in our comparisons.

Due to the high space requirement, we made a modification in the GRU-GCN~\cite{grugcn}, not applying the edge RNN. Instead, we used the same base RNN as our RTRGN model (see Appendix~\ref{Appendix:A}). This modification ensures that the experiments could be conducted within the available computational resources. 

\subsection{Dataset Summaries and New Datasets We Introduced}
\begin{table}[ht]
\caption{Link prediction datasets statistics.}
\scalebox{0.8}{
\begin{tabular}{c|cccccccc}
\hline
&              MovieLens     & Wikipedia     & Reddit        & SocialE.-1m   & SocialE.    & UCI-FSN       & Ubuntu        & Ecommerce        \\ \hline
 Nodes              & 2,625 & 9,227 & 10,985 & 71 & 74 & 1,899 & 159,316 & 35,571 \\
 Temporal links              & 100,000 & 157,474 & 672,447 & 176,090 & 2,099,519 &  59,835 & 964,437 & 414,812 \\
 Attributes            & 0 \& 0  & 172 \& 172& 172 \& 172 & 0 \& 0 & 0 \& 0 & 0 \& 0 & 0 \& 0 & 0 \& 0 \\
 Average Degree              &38.1 & 17.1 & 61.2 & 2480.1 & 28371.9 & 31.5 & 6.1 & 11.7 
 \\
 Bipartite              & True & True & True & False & False & False & True & True  \\
\hline
\end{tabular}}
\label{tab:datasets}
\end{table}

\textbf{Benchmark datasets}. We transformed the node and edge features of the Wikipedia and Reddit datasets into 172-dimensional feature vectors. However, since the other datasets are non-attributed, their features were set to zero. Table~\ref{tab:datasets} provides a summary of the benchmark datasets used in our experiments.

Among the benchmark datasets commonly used in temporal graph modeling are Wikipedia, Reddit, SocialE-1m, SocialE, UCI-FSN (often referred to as UCI in previous works \cite{nat,pint}), and Ubuntu. Detailed descriptions of these datasets are already given by previous works (e.g. can be found in Appendix B of \cite{nat}) thus is neglected here for simplicity. It is worth noting that the downloadable links for UCI-FSN are often misgiven, and the correct data source can be found in the Facebook-like Social Network dataset\footnote{\url{https://toreopsahl.com/datasets/\#online_forum_network}.}.

The \textbf{MovieLens} datasets\footnote{\url{https://grouplens.org/datasets/movielens/100k/}.} were collected by the GroupLens Research Project at the University of Minnesota. The dataset contains 100,000 ratings (ranging from 1 to 5) given by 943 users to 1,682 movies. In our study, we consider these ratings as links between the users and the corresponding movies while the score values are not used. 

The \textbf{Ecommerce} dataset is a collection of user behavior data from an online Ecommerce platform. The data was collected over a period of one month, starting from February 2023. The dataset focuses on hot users and hot items. Specifically, users with more than 50 interactions were randomly sampled, and items that had at least 10 interactions with the sampled set of users were retained. As a result, the dataset consists of approximately 11,920 users and 23,651 items. Each link in the dataset represents a user action, such as a click or purchase.

\textbf{Oscillating CSL}. It is an extension of the circular skip link (CSL) graph, which is a widely used benchmark for evaluating the expressiveness of static graph neural networks (GNNs). It is illustrated in Figure~\ref{fig:OscillatingCSL} and detailed discussed in the main paper. In our experiments, we sample prototypes of CSL from $\{\mathbf{C}_{11,2},\ldots,\mathbf{C}_{11,5}\}$. Note that the Oscillating CSL dataset also differs fundamentally from the Dynamic CSL task proposed in \cite{grugcn}. 
In the Oscillating CSL dataset, the stacking form of the temporal graphs is indistinguishable by the Temporal-1WL test. By contrast, in the Dynamic CSL dataset, the stacking form of the temporal graphs can be distinguished by the Temporal-1WL test, which makes it much easier. 

\textbf{Cuneiform Writing}. Cuneiform graph dataset\footnote{\url{ https://git.litislab.fr/bgauzere/py-graph/-/tree/e6c4229b23c82a0b27d6b2ecdab1efa420bdac6a/datasets/Cuneiform}.} described in~\cite{cuneiform} is a publicly available dataset comprising 267 graphs. These graphs represent 29 different Hittite cuneiform signs. On average, the graphs have 21.3 nodes and an average degree of 44.8.
Initially, the dataset contains graphs with features associated with the nodes and edges. However, for the purpose of this study, the attributes and features were removed, otherwise there will be no negative examples, i.e. isomorphic graphs. This results in graphs that represent 6 types of non-isomorphic static graphs.
It is worth noting that the original order of the edges in the dataset is identical for any two graphs with the same number of nodes. Thus the original ordering does not give rise to additional non-isomorphic temporal graphs. 

The Cuneiform dataset consists of two types of edges: the wedge and the arrangement. The wedge edges always create a fully-connected tetrahedron structure, while the arrangement edges establish dense connections between the center nodes of each tetrahedron. 
To achieve a balance between the number of isomorphic and non-isomorphic graph pairs: most of the densely connected arrangement edges were removed, retaining only edges corresponding to short spatial distances, which can be computed from the given node attributes representing their spatial location; each center node is limited to a maximum of two arrangement edges. This modification increases the number of non-isomorphic static graphs to 19. 

To convert they into temporal graphs, the edges were generated in a sequential order. Initially, all arrangement edges are added simultaneously at time $t=0$, followed by the addition of wedge edges in groups corresponding to the tetrahedrons at subsequent time steps ($t=1,2,\ldots$). 


\newpage

\section{Appendix: More Experiment Results}
\label{appendix: Experiment Results}

\subsection{AUC of Temporal Link Prediction}
We present the AUC metric results for the temporal link prediction tasks conducted in the main paper. The results are shown in Table \ref{tab:main auc} and \ref{tab:main 2 layer auc}, which exhibit a consistent trend similar to the AP metric.
\vspace{0mm}
\begin{table}[ht]
\centering
\caption{AUC ($\%$) for temporal link prediction tasks. Reported results are averaged over 10 runs. }
\scalebox{0.68}{
\begin{tabular}{c|c|ccccccccc}
\hline
& Model              & MovieLens     & Wikipedia     & Reddit        & SocialE.-1m   & SocialE.      & UCI-FSN       & Ubuntu        & Ecommerce \\ \hline
\parbox[t]{2mm}{\multirow{7}{*}{\rotatebox[origin=c]{90}{Transductive}}} 
& Dyrep              & $80.54\pm0.3$ & $95.34\pm0.2$ & $97.93\pm0.2$ & $83.91\pm0.7$ & $91.25\pm0.1$ & $60.20\pm2.2$ & $84.82\pm0.2$ & $61.59\pm0.3$ \\
& Jodie              & $80.97\pm0.3$ & $96.05\pm0.4$ & $97.63\pm0.1$ & $77.48\pm2.0$ & $81.00\pm0.4$ & $89.09\pm1.0$ & $91.08\pm0.2$ & $66.80\pm0.5$ \\
& TGAT               & $72.01\pm0.3$ & $96.53\pm0.1$ & $98.58\pm0.1$ & $51.32\pm0.5$ & $50.08\pm0.1$ & $76.27\pm0.5$ & $79.40\pm0.3$ & $60.99\pm0.5$ \\
& TGN           & $86.00\pm0.5$ & $97.28\pm0.2$ & $98.24\pm0.2$ & $91.73\pm0.4$ & $92.73\pm1.4$ & $85.47\pm0.3$ & $88.60\pm0.2$ & $80.03\pm0.5$ \\
& CAWN               & $80.58\pm0.4$ & $97.96\pm0.2$ & $97.37\pm0.2$ & $84.27\pm0.4$ & $84.24\pm0.6$ & $86.76\pm0.3$ & -             & $80.07\pm0.1$ \\
& NAT                & $78.85\pm2.0$ & $97.99\pm0.2$ & $98.58\pm0.2$ & $87.32\pm0.6$ & $92.90\pm0.9$ & $91.63\pm0.7$ & $88.98\pm0.5$ & $76.24\pm0.4$ \\
& GRU-GCN            & $86.77\pm0.1$ & $96.17\pm0.1$ & $97.95\pm0.3$ & $88.63\pm0.7$ & $87.86\pm1.8$ & $90.79\pm0.4$ & $86.63\pm0.1$ & $78.05\pm0.3$ \\
& PINT               & $85.65\pm0.6$ & $98.14\pm0.1$ & $98.40\pm0.1$ & $86.52\pm1.4$ & $88.93\pm2.4$ & $90.84\pm0.3$ & $89.10\pm0.2$ & $78.65\pm0.1$ \\
& RTRGN (ours)   & \textcolor{violet}{$\mathbf{88.00\pm0.2}$} 
                                     & \textcolor{violet}{$\mathbf{98.48\pm0.2}$} 
                                                     & \textcolor{violet}{$\mathbf{98.92\pm0.2}$} 
                                                                     & \textcolor{violet}{$\mathbf{93.90\pm0.3}$} & \textcolor{violet}{$\mathbf{95.92\pm0.2}$} & \textcolor{violet}{$\mathbf{95.98\pm0.1}$} 
                                                                                                                     & \textcolor{violet}{$\mathbf{96.18\pm0.1}$}
                                                                                                                                     & \textcolor{violet}{$\mathbf{86.05\pm0.4}$} \\ \hline
\parbox[t]{2mm}{\multirow{7}{*}{\rotatebox[origin=c]{90}{Inductive}}}   
& Dyrep              & $76.64\pm0.3$ & $93.97\pm0.2$ & $96.86\pm0.3$ & $76.48\pm1.9$ & $89.13\pm1.0$ & $52.11\pm1.3$ & $67.12\pm0.3$ & $53.29\pm0.6$ \\
& Jodie              & $76.92\pm0.2$ & $95.16\pm0.5$ & $96.31\pm0.2$ & $79.50\pm1.5$ & $85.59\pm0.4$ & $73.91\pm1.2$ & $80.78\pm0.3$ & $58.03\pm0.4$ \\
& TGAT               & $71.59\pm0.8$ & $97.52\pm0.2$ & $96.47\pm0.1$ & $51.62\pm0.6$ & $50.09\pm0.1$ & $69.90\pm0.7$ & $76.71\pm0.3$ & $58.88\pm0.9$ \\
& TGN           & $83.88\pm0.5$ & $96.78\pm0.2$ & $96.69\pm0.2$ & $91.40\pm0.3$ & $91.72\pm1.2$ & $79.29\pm0.2$ & $77.83\pm0.3$ & $77.82\pm0.5$ \\
& CAWN               & $74.86\pm0.5$ & $97.33\pm0.2$ & $96.81\pm0.2$ & $75.48\pm0.4$ & $82.34\pm0.2$ & $87.06\pm0.3$ & -             & $77.76\pm0.2$  \\
& NAT                & $79.66\pm0.5$ & $97.44\pm0.5$ & $97.10\pm0.9$ & $86.63\pm1.5$ & $89.23\pm1.7$ & $85.36\pm0.3$ & $76.71\pm0.6$ & $69.12\pm0.2$ \\
& GRU-GCN            & $84.85\pm0.2$ & $93.25\pm0.2$ & $96.17\pm0.3$ & $87.52\pm0.9$ & $84.05\pm4.0$ & $85.37\pm0.6$ & $71.85\pm0.6$ & $76.01\pm0.2$ \\
& PINT               & $83.45\pm0.5$ & $97.58\pm0.1$ & $97.68\pm0.1$ & $78.18\pm4.1$ & $78.92\pm2.1$ & $87.44\pm0.3$ & $84.74\pm0.2$ & $74.83\pm0.2$ \\
& RTRGN (ours)   & \textcolor{violet}{$\mathbf{86.43\pm0.2}$} & \textcolor{violet}{$\mathbf{98.05\pm0.2}$} 
                                                     & \textcolor{violet}{$\mathbf{98.00\pm0.1}$} & \textcolor{violet}{$\mathbf{93.74\pm0.4}$} 
                                                                     & \textcolor{violet}{$\mathbf{94.81\pm0.5}$} & \textcolor{violet}{$\mathbf{93.09\pm0.1}$} 
                                                                                                                     & \textcolor{violet}{$\mathbf{91.19\pm0.2}$} 
                                                                                                                                     & \textcolor{violet}{$\mathbf{84.37\pm0.5}$} \\ \hline
\end{tabular}}
\label{tab:main auc}
\end{table}
\vspace{-5mm}
\begin{table}[ht]
\centering
\caption{AUC ($\%$) for temporal link prediction tasks performed on 2-layer models. \textcolor{violet}{Best} performing methods are highlighted. Reported results are averaged over 10 runs.}
\scalebox{0.7}{
\begin{tabular}{c|c|cccc}
\hline
AUC & Model          & MovieLens     & Wikipedia     & UCI-FSN       & Ecommerce  \\ \hline
\parbox[t]{2mm}{\multirow{4}{*}{\rotatebox[origin=c]{90}{Trans}}} 
& TGN           & $86.72\pm0.5$ & $98.25\pm0.2$ & $85.54\pm0.3$ & $82.52\pm0.2$ \\
& NAT                & $82.63\pm0.9$ & $97.90\pm0.2$ & $91.61\pm0.3$ & $77.98\pm0.5$ \\
& PINT               & $86.09\pm0.2$ & $98.00\pm0.1$ & $91.68\pm0.3$ & $79.61\pm0.7$ \\
& RTRGN (ours)        & \textcolor{violet}{$\mathbf{96.13\pm0.2}$}
                                     & \textcolor{violet}{$\mathbf{98.56\pm0.2}$} 
                                                     & \textcolor{violet}{$\mathbf{98.21\pm0.1}$} 
                                                                     & \textcolor{violet}{$\mathbf{93.91\pm0.1}$} \\ \hline
\parbox[t]{2mm}{\multirow{4}{*}{\rotatebox[origin=c]{90}{Ind}}}
& TGN           & $85.08\pm0.5$ & $97.68\pm0.2$ & $80.34\pm0.3$ & $80.98\pm0.2$ \\
& NAT                & $72.26\pm1.5$ & $97.41\pm0.2$ & $85.41\pm1.0$ & $72.49\pm2.4$ \\
& PINT               & $83.73\pm0.4$ & $96.56\pm0.2$ & $88.61\pm0.4$ & $77.01\pm0.3$ \\
& RTRGN (ours)        & \textcolor{violet}{$\mathbf{94.89\pm0.3}$} 
                                     & \textcolor{violet}{$\mathbf{98.03\pm0.5}$}
                                                     & \textcolor{violet}{$\mathbf{95.95\pm0.1}$} 
                                                                     & \textcolor{violet}{$\mathbf{91.48\pm0.1}$} \\ \hline
\end{tabular}}
\label{tab:main 2 layer auc}
\end{table}

\subsection{Ablation Studies on Sampling (the number of most recent neighbors $d$)}
\begin{figure}[ht]
    \centering
    \includegraphics[height=36mm]{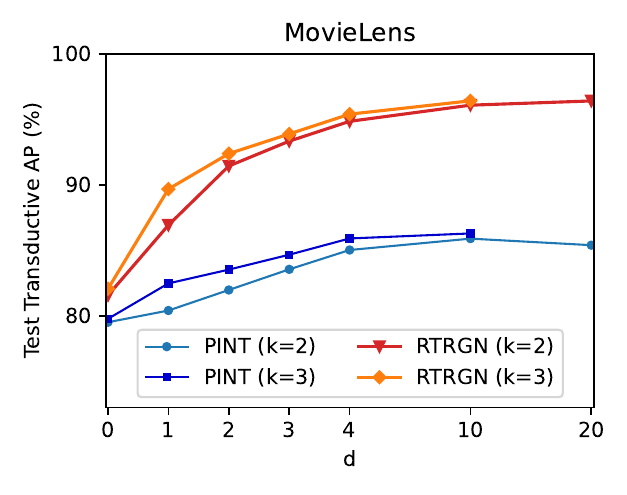}
    \includegraphics[height=36mm]{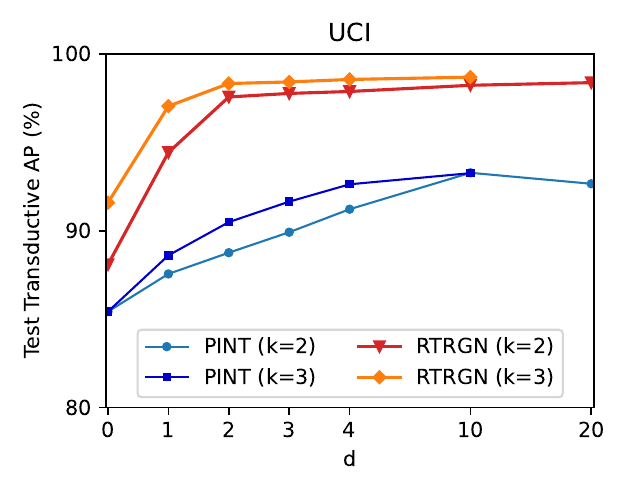}
    \includegraphics[height=36mm]{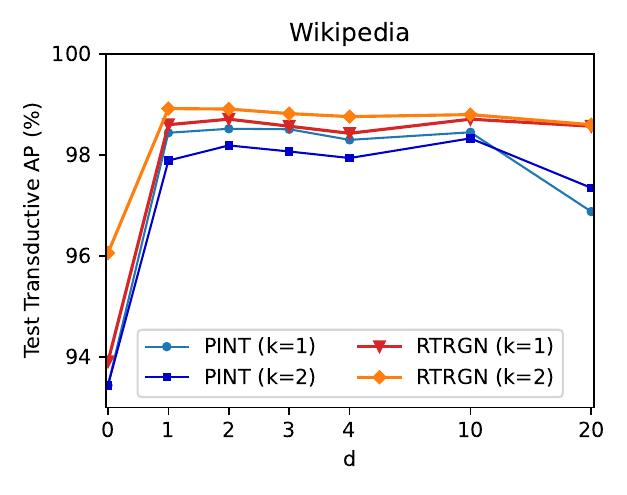}
    \caption{Mean test transductive AP ($\%$) for temporal link prediction tasks performed on three datasets w.r.t. varying $d$. Reported results are averaged over 10 runs. Here, $d=0$ represents the model consisting of pure RNN, without further considering any recent neighbors in the aggregation process.}
    \label{fig:d ablation}
\end{figure}
We conducted an ablation study on the sampling, the number of most recent neighbors $d$, and present the results\footnote{The results for $d=20$ are omitted for $k=3$ due to potential memory constraints. In future work, we plan to handle the memory issue by removing repeated computations of neighbors. For the Wikipedia dataset, we consider models with up to 2 layers for efficiency purposes.} in Figure~\ref{fig:d ablation}. The optimal value of $d$ varies for different datasets and exhibits different trends with varying $d$. 

For MovieLens, increasing $d$ consistently improves performance. This can be attributed to two factors: (1) Users in MovieLens only interact with a movie once, so including more neighbors provide unbiased and comprehensive information. (2) The average degree in MovieLens is higher compared to the other datasets. 

In UCI, the performance quickly converges for RTRGN as $d$ increases. This is because in such a social network dataset, nodes constantly interact with the same neighbors, making a smaller $d$ sufficient for representation.

In Wikipedia, we observe a performance drop as $d$ increases. This may be due to a longer length of recent interactions that introduces harmful patterns for the model. Additionally, it is worth noting that PINT exhibits a performance drop when transitioning from $k=1$ to $k=2$ (similarly observed for TGN but not presented here). Such observation, together with further ablation studies on Wikipedia reveals its high sensitivity to sequence information: GNN-like aggregations practically cannot effectively capture the time and sequence information, while sequential state updaters can more effectively capture such information, resulting in better performance (Appendix~\ref{appendix:blurring}). 

We also noticed a consistent drop in performance for PINT as $d$ increases from 10 to 20. Similar observations were found in TGN but are not shown here. This suggests that the baseline models may struggle to model long sequences of the most recent neighbors. In comparison, RTRGN only exhibits a marginal performance drop when increasing $d$ on Wikipedia. This can be attributed to the properties of the Wikipedia dataset itself rather than a drawback in algorithm design. The robustness of RTRGN to changes in $d$ can be attributed to its formulation, which includes hidden states and sequential updates that more effectively maintain sequential information of the events (Appendix~\ref{appendix:blurring}). 

\subsection{More Ablation Studies Results}
\begin{figure}[ht]
    \centering
    \vspace{-3mm}
    \includegraphics[width=140mm]{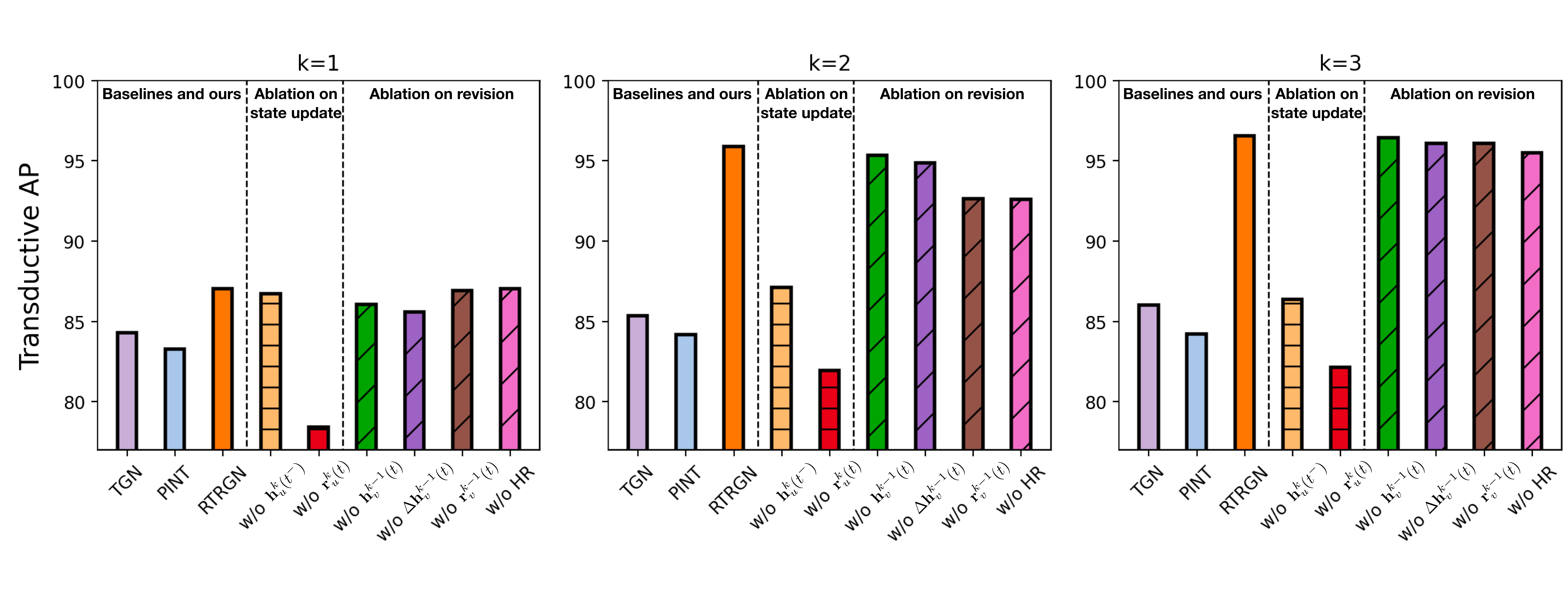}
    \vspace{-10mm}
    \caption{Transductive AP (\%) on MovieLens dataset. Results averaged over 10 runs.}
    \label{fig:main ablation movie trans ap}
\end{figure} 
\begin{figure}[ht]
    \centering
    \vspace{-5mm}
    \includegraphics[width=140mm]{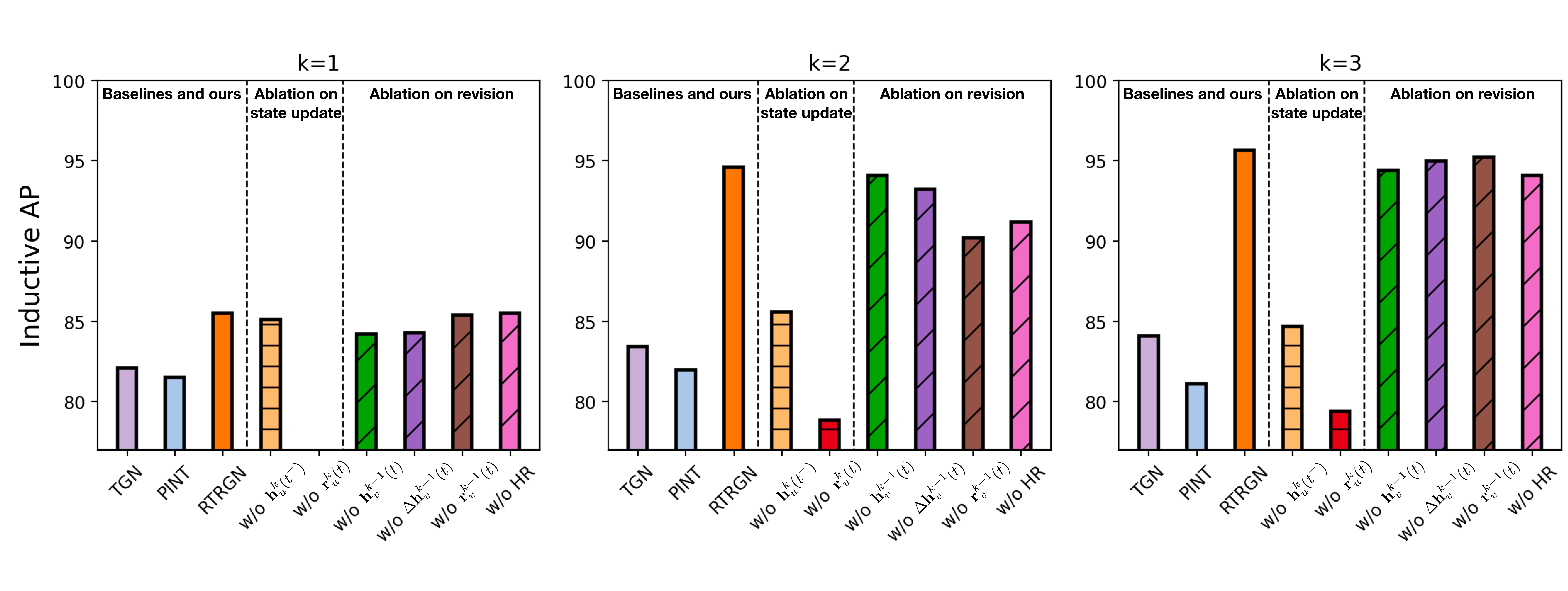}
    \vspace{-10mm}
    \caption{Inductive AP (\%) on MovieLens dataset. Results averaged over 10 runs.}
    \label{fig:main ablation movie ind ap}
\end{figure} 
\begin{figure}[ht]
    \centering
    \vspace{-5mm}
    \includegraphics[width=140mm]{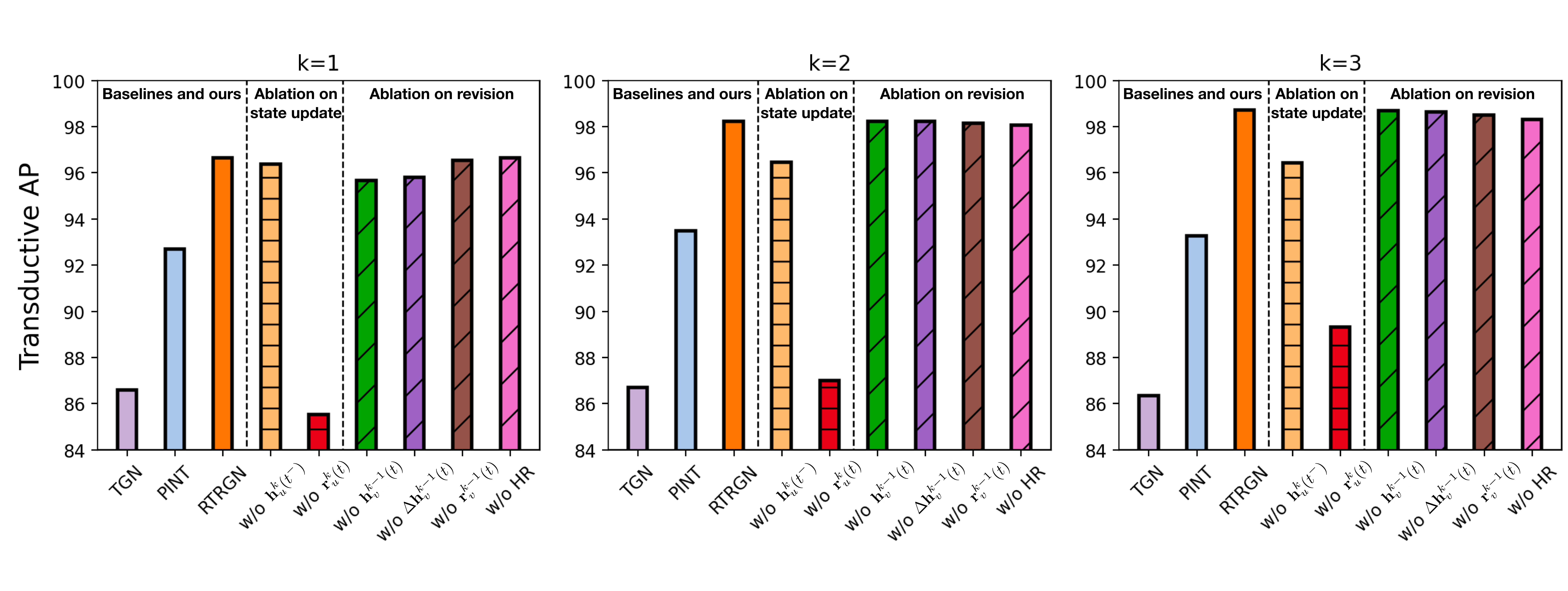}
    \vspace{-10mm}
    \caption{Transductive AP (\%) on UCI dataset. Results averaged over 10 runs.}
    \label{fig:main ablation uci trans ap}
\end{figure} 

The ablation study settings we considered in the main paper includes:

\begin{itemize}
    \item w/o $\mathbf{h}_{u}^{k}(t^{-})$: where we set $\mathbf{h}_{u}^{k}(t^{-})=\mathbf{0}$ in \eqref{equ:TRGN_recursion} for $k>0$, hence do not retain the hidden state but compute aggregation dynamically. Besides setting it to zero, we also tried use another design as $\mathbf{h}_{u}^{k}(t)=\textsf{MLP}^{k}\big(\mathbf{h}_{u}^{k-1}(t),\mathbf{h}_{v}^{k-1}(t),\mathbf{r}_{u}^{k}(t),\Phi(e_{uv,t})\big)$ to exclude the higher-level state cache, for $k>0$. 
    Both settings gives highly consistent results.
    \item w/o $\mathbf{r}_{u}^{k}(t)$: where we set $\mathbf{r}_{u}^{k}(t)=\mathbf{0}$ in \eqref{msg}, which leads to a pure RNN design without neighbor aggregation. Such an RNN can have multiple layer though.
    \item w/o $\mathbf{h}_{v}^{k-1}(t)$: we remove $\mathbf{h}_{v}^{k-1}(t)$ in \eqref{equ:revision_recursion}.
    \item w/o $\Delta\mathbf{h}_{v}^{k-1}(t)$: we remove $\Delta\mathbf{h}_{v}^{k-1}(t)$ in \eqref{equ:revision_recursion}.
    \item w/o $\mathbf{r}_{v}^{k-1}(t)$: we remove $\mathbf{r}_{v}^{k-1}(t)$ in \eqref{equ:revision_recursion}.
    \item w/o HR: we use the homogeneous version of revision computation \eqref{equ:revision_recursion no H} instead of the heterogeneous version \eqref{equ:revision_recursion}.
\end{itemize}

As observed in Figure \ref{fig:main ablation movie trans ap} and Figure \ref{fig:main ablation movie ind ap}, the trends of transductive AP and inductive AP are highly consistent with the transductive AUC results shown in Figure \ref{fig:main ablation movie} of the main paper. Similar ablation studies were also conducted on the UCI dataset, and the results are presented in Figure \ref{fig:main ablation uci trans ap}.

The findings of the ablation studies support the conclusions drawn in the main paper, which can be summarized as follows:
(1) The significant improvement observed when $k > 1$ is mainly due to the inclusion of higher-level {hidden} states; (2) Integrating neighbor information, including revision, in the message plays a critical role in improving performance; (3) The effectiveness of revision relies primarily on the state changes of neighbors, rather than the neighbor states themselves; (4) The results highlight the benefits of incorporating neighbor revisions in the temporal revision module and the introduction of heterogeneity in the revision process. 

\subsection{Extended Ablation Studies on the Model Design}
\begin{figure}[ht]
    \centering
    \vspace{-3mm}
    \includegraphics[width=140mm]{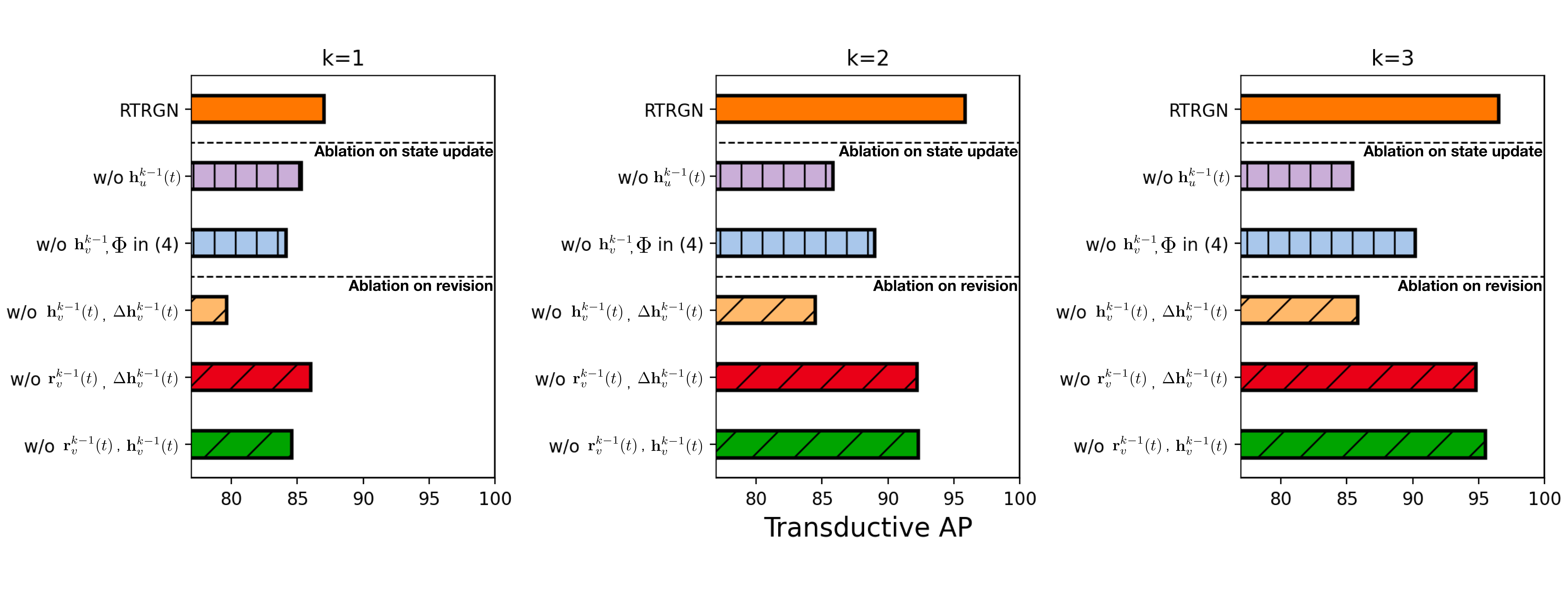}
    \vspace{-10mm}
    \caption{Transductive AP (\%) on MovieLens dataset. Results averaged over 10 runs.}
    \label{fig:more ablation movie trans ap}
\end{figure} 
\begin{figure}[ht]
    \centering
    \vspace{-3mm}
    \includegraphics[width=140mm]{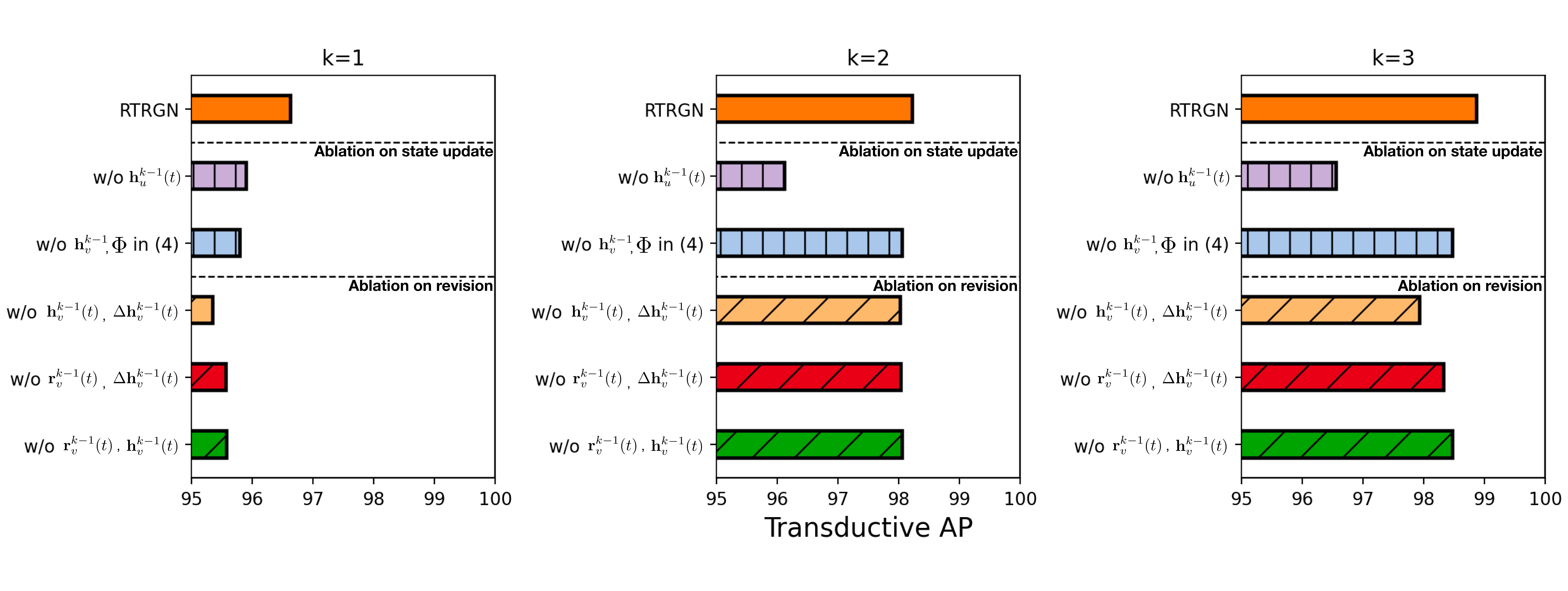}
    \vspace{-10mm}
    \caption{Transductive AP (\%) on UCI dataset. Results averaged over 10 runs.}
    \label{fig:more ablation uci trans ap}
\end{figure} 

We further examine the following settings as additional ablation studies:
\begin{itemize}
    \item w/o $\mathbf{h}_{u}^{k-1}(t)$: we remove $\mathbf{h}_{u}^{k-1}(t)$ in \eqref{msg}.
    \item w/o $\mathbf{h}_{v}^{k-1}(t)$ and $\Phi(e_{uv,t})$ in \eqref{msg}: we remove $\mathbf{h}_{v}^{k-1}(t)$ and $\Phi(e_{uv,t})$ in \eqref{msg}.
    \item w/o $\mathbf{h}_{v}^{k-1}(t)$ and $\Delta\mathbf{h}_{v}^{k-1}(t)$: we remove both $\mathbf{h}_{v}^{k-1}(t)$ and $\Delta\mathbf{h}_{v}^{k-1}(t)$ in \eqref{equ:revision_recursion}, and set the base case to $\mathbf{r}_{u}^{0}(t)=\Delta\mathbf{h}_{u}^{0}(t)$.
    \item w/o $\mathbf{r}_{v}^{k-1}(t)$ and $\Delta\mathbf{h}_{v}^{k-1}(t)$: we remove both $\mathbf{r}_{v}^{k-1}(t)$ and $\Delta\mathbf{h}_{v}^{k-1}(t)$ in \eqref{equ:revision_recursion}.
    \item w/o $\mathbf{r}_{v}^{k-1}(t)$ and $\mathbf{h}_{v}^{k-1}(t)$: we remove both $\mathbf{r}_{v}^{k-1}(t)$ and $\mathbf{h}_{v}^{k-1}(t)$ in \eqref{equ:revision_recursion}.
\end{itemize}
The results are presented in Figure \ref{fig:more ablation movie trans ap} and Figure \ref{fig:more ablation uci trans ap}.

\textbf{More ablation studies on the state update:} 


Firstly, we find that the link to the lower-level state $\mathbf{h}_{u}^{k-1}(t)$ is also critical. Removing this link leads to a significant drop in performance.

Secondly, when removing the latest event from the message, i.e., excluding $\mathbf{h}_{v}^{k-1}(t)$ in \eqref{msg}, we observe a significant drop in performance on MovieLens, suggesting its importance. A similar phenomenon is observed in the Wikipedia dataset, as we will demonstrate later. 
However, in UCI, the drop is less pronounced. This is likely because nodes in the UCI dataset tend to have more consistent interactions with the same neighbors, making them more robust to the removal of just one most recent neighbor.

\textbf{More ablation studies on the revision:} 

Compared to the previous ablation studies settings on the revision, we now remove two terms together and only keep one term in the computation of revision in \eqref{equ:revision_recursion}. 

When we remove both  $\mathbf{h}_{v}^{k-1}(t)$ and $\Delta\mathbf{h}_{v}^{k-1}(t)$ terms from the computation of revision in \eqref{equ:revision_recursion}, the performance can drop significantly. This is because the recursion becomes independent of the high level neighbor states and their state changes. 
This suggests that we should maintain at least one of the terms in  $\mathbf{h}_{v}^{k-1}(t)$ and $\Delta\mathbf{h}_{v}^{k-1}(t)$. 

In the other two settings (w/o $\mathbf{r}_{v}^{k-1}(t)$ and $\Delta\mathbf{h}_{v}^{k-1}(t)$, and w/o $\mathbf{r}_{v}^{k-1}(t)$ and $\mathbf{h}_{v}^{k-1}(t)$), the performance difference indicates that $\Delta\mathbf{h}_{v}^{k-1}(t)$ may be more important than $\mathbf{h}_{v}^{k-1}(t)$ in our RTRGN model. This finding is generally consistent across the previous results (Figure \ref{fig:main ablation movie trans ap}, Figure \ref{fig:main ablation movie ind ap}, Figure \ref{fig:main ablation movie}, and Figure \ref{fig:main ablation uci trans ap}). Additionally, when compared with previous results, we observe that including a self-recursion $\mathbf{r}_{v}^{k-1}(t)$ always benefits the performance.

Notably, the setting, w/o $\mathbf{r}_{v}^{k-1}(t)$ and $\Delta\mathbf{h}_{v}^{k-1}(t)$, is the closest approximation of the classic GNN-like temporal neighbor aggregation module within our revision module. Despite this, it still largely outperforms the baselines, including TGN and PINT. This once again suggests the advantage of incorporating the hidden states, even in classic GNN-like temporal neighbor aggregation modules, different from our carefully designed revision module.

\subsection{Ablation Studies: Sequential Updates Better Captures Time Series Information}
\label{appendix:blurring}

The potential issue with GNN aggregation is that it may not be able to capture temporal and sequential information well in its generated embeddings. Temporal GNNs typically operate on graphs where only the node/edge attributes encode time and sequence, but the model structure itself is not aware of time. Though a theoretically perfect GNN can potentially capture temporal information effectively, it remains unclear whether existing practical implementations have limitations in this regard. That is, the practical ability of GNNs to utilize temporal information is uncertain. In cases where temporal and sequential information is critical but not adequately modeled, the performance may be compromised.

To investigate the potential issue of GNNs not effectively capturing time series information and assessing the robustness of our RTRGN in such scenarios, we created a subsampled Wikipedia dataset. This dataset only includes Wikipedia entries that exhibit a clear "switch of editor" pattern: where for each of such entries there is more than one user, each interaction only for a short period rather than only having a single user throughout their entire interaction history. The resulting dataset, which contains 50,132 interactions and 5,616 unique nodes, represents approximately one-third of the total interactions and $60\%$ of the total nodes in the original Wikipedia dataset.\footnote{We will make the subsampled Wikipedia dataset publicly available alongside our code.} We chose Wikipedia for this study because our ablation studies on the parameter $d$ (Figure \ref{fig:d ablation}) suggested that the issue of GNNs not effectively capturing time series information might be particularly pronounced in the Wikipedia dataset, as increasing $d$ consistently resulted in a decline in performance. 

We hypothesize that in the subsampled Wikipedia dataset, where temporal and sequential information plays a crucial role, GNN aggregation may struggle to effectively capture these dynamics. The results of our experiments are presented in Table~\ref{tab:gnn_blurring}, which verifies our conjecture. It is evident that the inclusion of the GNN module ($k=1,2$) significantly impairs the performance of the baseline models ($k=0$ which does not involve graph operations). Note that our base RNN implementation outperforms the implementation of baselines  (such as the base RNN in TGN or PINT), so we also adjusted the baselines TGN or PINT to use the same base RNN as ours (Appendix~\ref{Appendix:A}) for a fair comparison. Even with this adjustment, the drop in performance is still evident, indicating a consistent trend. PINT shows slightly better performance than TGN, which could be attributed to its more advanced time-injective aggregation and update mechanisms.

RTRGN is the only method that outperforms its vanilla base RNN ($k=0$), indicating the effectiveness of our RNN framework. 
The results may drop with increasing $k$ if the hidden state is not included (denoted as w/o $\mathbf{h}_{u}^{k}(t^{-})$ for $k\geq1$). In this case, the result for $k\geq1$ is not computed from a state transition function on the hidden state, but rather computed on the fly. On the contrary, when including the hidden state, the sequential update formulation effectively captures these dynamics, leading to consistent better performance with increasing $k$. Further ablations which removes the event information in message (w/o $\mathbf{h}_{v}^{k-1}(t)$ and $\Phi(e_{uv,t})$ for $k\geq1$ in \eqref{msg}) verifies its effectiveness in the performance. 

Additionally, it is surprising that our RTRGN still benefits from its revision module. This suggests that it, together with the hidden states, is able to effectively capture time and sequence information. 

\begin{table}[ht]
\centering
\caption{Transductive AP / inductive AP on subsampled Wikipedia. The results better than vanilla base RNN ($k=0$) are bold.}
\scalebox{0.75}{
\begin{tabular}{l|c|c|c}
\hline
Subsampled Wikipedia                                    & $k=0$ (Only base RNN)  & $k=1$                  & $k=2$                         \\ \hline
TGN                                                     & 57.59\hspace{4mm}/\hspace{4mm}56.08          & 56.12\hspace{4mm}/\hspace{4mm}54.72          & 50.65\hspace{4mm}/\hspace{4mm}50.15 \\
TGN  (with our base RNN)                                & 88.23\hspace{4mm}/\hspace{4mm}86.04          & 61.82\hspace{4mm}/\hspace{4mm}60.03          & 62.83\hspace{4mm}/\hspace{4mm}60.32 \\
PINT (with our base RNN)                                & 88.23\hspace{4mm}/\hspace{4mm}86.04          & 63.45\hspace{4mm}/\hspace{4mm}60.77          & 65.49\hspace{4mm}/\hspace{4mm}62.41 \\
RTRGN                                                    & 88.23\hspace{4mm}/\hspace{4mm}86.04          & \textbf{90.52}\hspace{4mm}/\hspace{4mm}\textbf{89.51}          & \textbf{91.96}\hspace{4mm}/\hspace{4mm}\textbf{91.41} \\ \hline
RTRGN (w/o $\mathbf{h}_{u}^{k}(t^{-})$ for $k\geq1$)     & 88.23\hspace{4mm}/\hspace{4mm}86.04          & 86.56\hspace{4mm}/\hspace{4mm}84.03          & 72.71\hspace{4mm}/\hspace{4mm}66.35 \\ 
RTRGN (w/o $\mathbf{r}_{u}^{k}(t)$ for $k\geq1$)                      & 88.23\hspace{4mm}/\hspace{4mm}86.04          & \textbf{89.36}\hspace{4mm}/\hspace{4mm}\textbf{86.12}          & \textbf{89.06}\hspace{4mm}/\hspace{4mm}85.68 \\ 
RTRGN (w/o $\mathbf{h}_{v}^{k-1}(t)$ and $\Phi(e_{uv,t})$ for $k\geq1$ in \eqref{msg})                      & 88.23\hspace{4mm}/\hspace{4mm}86.04          & 88.02\hspace{4mm}/\hspace{4mm}85.96          & 87.73\hspace{4mm}/\hspace{4mm}85.42 \\ \hline
\end{tabular}}
\label{tab:gnn_blurring}
\end{table}

\subsection{Practical Run-time Comparison}
\label{appendix: time complexity}

We compare the practical run-time of our RTRGN with baselines, including CAWN, TGAT, TGN, NAT, GRU-GCN, PINT and PINT-pos. 
Among these models, CAWN and PINT-pos are found to be much slower than other methods, due to their inefficient construction of neighbor structural features. 

Table~\ref{tab:main time} shows that RTRGN actually can run faster than TGN and PINT, thanks to our efficient implementation of the base RNN and the recursive computation. NAT is the fastest among all tested methods, while CAWN is the slowest. It's worth noting that the original public implementation of CAWN has a reported bug when using GPUs, so its experiments were conducted with CPU only. 

\begin{table}[ht]
\centering
\caption{\textit{Averaged time per epoch | averaged total time} for temporal link prediction tasks on layer 1 and layer 2 models.}
\scalebox{0.7}{
\begin{tabular}{cc|cccccc}
\hline
& Layer 1 Model      & MovieLens     & Wikipedia     & Reddit        & SocialE.      & Ubuntu        & Ecommerce \\ \hline
& TGN                & 58.3 | 1340  & 66.2 | 1655   & 207.8 | 3262  & 154.0 | 2125  & 2363 | 2363   & 430.6 | 12444 \\
& CAWN               & 296.0 | 1687 & 1.6$*10^4$ | 1.6$*10^5$ & 5026 | 7.2$*10^4$
                                                                     & 8.9$*10^4$ | 4.4$*10^5$   & - & 1540  | 7702  \\
& NAT                & 10.8 | 284   & 22.9 | 477    & 48.6  | 1374  & 281.2 | 7582  & 58.1 | 907    & 27.6  | 1022  \\
& PINT               & 70.2 | 1785  & 88.0 | 1878   & 340.4 | 4459  & 202.9 | 1846  & 3985 | 8368   & 643.0 | 7523  \\
& RTRGN (ours)       & 12.3 | 270   & 27.7 | 706    & 155.6 | 3134  & 145.6 | 1426  & 2702 | 2702   & 296.4 | 5192  \\ \hline
\end{tabular}}

\scalebox{0.7}{
\begin{tabular}{cc|ccccc}
\hline
& Layer 2 Model       & MovieLens     & Wikipedia     & UCI            & Ecommerce        \\ \hline
& TGN                 & 73.6 | 2307   & 253.0 | 6122  &  43.5 | 970    & 532.1 | 10812  \\
& NAT                 &  24.7 | 421   &  32.9 | 667   &   9.6 | 340    &  62.9 | 1561  \\
& PINT                & 81.5 | 2599   & 280.0 | 6692  &  75.3 | 1402   & 915.7 | 10307   \\
& RTRGN (ours)        &  44.8 | 648   & 159.5 | 2530  &  25.9 | 721    & 773.4 | 4792  \\ \hline
\end{tabular}}
\label{tab:main time}
\end{table}


\subsection{Explicit and Implicit Base Embeddings}

We conducted an ablation study comparing the two implementations of base embeddings discussed, and the results indicate that the implicit base case generally outperforms the explicit base case. However, it is worth noting that the explicit base case runs faster due to its lower number of RNN computations.

\begin{table}[ht]
\centering
\caption{Average Precision ($\%$) for temporal link prediction tasks on both transductive and inductive settings. \textcolor{violet}{Best} performing method is highlighted. Reported results are averaged over 10 runs. }
\scalebox{0.65}{
\begin{tabular}{c|c|ccccccccc}
\hline
& Model              & MovieLens     & Wikipedia     & Reddit        & SocialE.-1m   & SocialE.      & UCI-FSN       & Ubuntu        & Ecommerce        \\ \hline
\parbox[t]{2mm}{\multirow{2}{*}{\rotatebox[origin=c]{90}{Trans}}} 
& RTRGN (Impl)   & \textcolor{violet}{$\mathbf{86.56\pm0.2}$}
                                     & \textcolor{violet}{$\mathbf{98.56\pm0.2}$} 
                                                     & \textcolor{violet}{$\mathbf{99.00\pm0.1}$}
                                                                     & $91.42\pm0.4$ & $93.52\pm0.3$ & \textcolor{violet}{$\mathbf{96.43\pm0.1}$} 
                                                                                                                     & \textcolor{violet}{$\mathbf{96.69\pm0.1}$}
                                                                                                                                     & \textcolor{violet}{$\mathbf{88.05\pm0.4}$} \\ 
& RTRGN (Expl)   & $86.31\pm0.2$ & $98.26\pm0.2$ & $98.70\pm0.1$ & \textcolor{violet}{$\mathbf{92.20\pm0.1}$}
                                                                                     & \textcolor{violet}{$\mathbf{94.02\pm0.3}$}
                                                                                                     & $95.25\pm0.1$ & $94.06\pm0.1$ & $84.65\pm0.1$             \\ \hline
\parbox[t]{2mm}{\multirow{2}{*}{\rotatebox[origin=c]{90}{Ind}}}   
& RTRGN (Impl)   & \textcolor{violet}{$\mathbf{85.03\pm0.2}$} 
                                     & \textcolor{violet}{$\mathbf{98.06\pm0.2}$}
                                                     & \textcolor{violet}{$\mathbf{98.26\pm0.1}$} 
                                                                     & $91.31\pm0.3$ & $90.15\pm0.8$ & \textcolor{violet}{$\mathbf{94.26\pm0.1}$} 
                                                                                                                     & \textcolor{violet}{$\mathbf{92.95\pm0.1}$} 
                                                                                                                                     & \textcolor{violet}{$\mathbf{86.91\pm0.5}$} \\ 
& RTRGN (Expl)   & $84.20\pm0.2$ & $97.38\pm0.2$ & $96.98\pm0.1$ & \textcolor{violet}{$\mathbf{91.73\pm0.7}$} 
                                                                                     & \textcolor{violet}{$\mathbf{92.47\pm0.3}$}
                                                                                                     & $92.13\pm0.1$ & $91.21\pm0.1$ & $81.59\pm0.2$ \\ \hline
\end{tabular}}
\label{tab:main ap base case}
\end{table}

\end{document}